\documentclass[12pt]{article}%

\usepackage{etoolbox}
\newtoggle{preprint}
\newtoggle{journal}
\newcommand{\preprint}[1]{\iftoggle{preprint}{#1}{}}
\newcommand{\journal}[1]{\iftoggle{journal}{#1}{}}

 \toggletrue{preprint}
 \togglefalse{journal}

\preprint{
\usepackage{
	amsmath,
	amsthm,
	amssymb,
	wrapfig,
	cases,
	mathtools,
	thmtools,
	array,
	bbm,
	bm,
	subfigure,
	makecell,
	esvect,
	mathrsfs,
	breqn,
	booktabs,
	upgreek,
	changepage,
	dsfont,
	fixme,
	listings,
	multirow,
	xargs,
	xstring,
	multicol,
	graphicx,
	float,
	color,
	algpseudocode,
	indentfirst,
	accents,
	ifthen,
	wasysym,
	tikz,
	pgf,
	lmodern,
	authblk,
	}

\usepackage[utf8]{inputenc}
\usepackage[ruled,vlined]{algorithm2e}
\usepackage[hyperfootnotes=false, hidelinks]{hyperref}

\newcommand{\manualendproof}{\hfill\qedsymbol\\[2mm]}

\usepackage{paralist}
\usepackage{enumitem}

\setlist[itemize]{noitemsep, topsep=-\topsep}

\usetikzlibrary{calc, positioning,shapes,arrows.meta,automata,positioning,backgrounds,fit}

\usepackage{xpatch}
\usepackage[capitalize,nameinlink]{cleveref}
\usepackage{crossreftools}
\pdfstringdefDisableCommands{%
    \let\Cref\crtCref
    \let\cref\crtcref
}

\hypersetup{%
    bookmarksnumbered, bookmarksopen=true, bookmarksopenlevel=1,%
}

\crefname{figure}{Figure}{Figures}
\crefname{subsection}{Section}{Sections}
\crefname{lemma}{Lemma}{Lemmas}
\crefname{corollary}{Corollary}{Corollaries}
\crefname{theorem}{Theorem}{Theorems}
\crefname{informal}{Informal Theorem}{Informal Theorems}
\crefname{example}{Example}{Examples}
\crefname{property}{Property}{Properties}
\crefname{assumption}{Assumption}{Assumptions}
\crefname{appassumption}{Assumption}{Assumptions}

\declaretheorem[name=Theorem]{theorem}

\declaretheorem[sibling=theorem,name=Proposition]{proposition}
\declaretheorem[sibling=theorem,name=Corollary]{corollary}

\declaretheorem[sibling=theorem,name=Definition]{definition}
\declaretheorem[name=Assumption]{assumption}
\declaretheorem[sibling=theorem,name=Assumption]{appassumption}
\declaretheorem[name=Assumption, numbered=no]{assumption*}

\declaretheorem[qed=$\triangleleft$,sibling=theorem,name=Remark]{remark}

\numberwithin{equation}{section}
\numberwithin{theorem}{section}
\numberwithin{lemma}{section}
\numberwithin{remark}{section}
\numberwithin{proposition}{section}
\numberwithin{definition}{section}
\numberwithin{conjecture}{section}
\numberwithin{example}{section}
\numberwithin{corollary}{section}
\numberwithin{appassumption}{section}

\makeatletter
\renewcommand{\maketag@@@}[1]{\hbox{\m@th\normalsize\normalfont#1}}%
\makeatother

\makeatletter
\let\reftagform@=\tagform@
\def\tagform@#1{\maketag@@@{\ignorespaces\textcolor{gray}{(#1)}\unskip\@@italiccorr}}
\renewcommand{\eqref}[1]{\textup{\reftagform@{\ref{#1}}}}
\makeatother

\newcommand{\EE}{\mathbb{E}}

\newcommand{\II}{\mathbb{I}}

\newcommand{\NN}{\mathbb{N}}

\newcommand{\PP}{\mathbb{P}}

\newcommand{\RR}{\mathbb{R}}

\newcommand{\Aa}{\mathcal{A}}
\newcommand{\Bb}{\mathcal{B}}

\newcommand{\Ff}{\mathcal{F}}
\newcommand{\Gg}{\mathcal{G}}

\newcommand{\Ll}{\mathcal{L}}

\newcommand{\Pp}{\mathcal{P}}
\newcommand{\Qq}{\mathcal{Q}}

\newcommand{\Ss}{\mathcal{S}}

\newcommand{\Xx}{\mathcal{X}}
\newcommand{\Yy}{\mathcal{Y}}
\newcommand{\Zz}{\mathcal{Z}}

\newcommand{\eps}{\varepsilon}

\def\[#1\]{\begin{equation}\begin{aligned}#1\end{aligned}\end{equation}}
\def\*[#1\]{\begin{equation*}\begin{aligned}#1\end{aligned}\end{equation*}}

\def\s*[#1\s]{\small\begin{align*}#1\end{align*}\normalsize}

\newcommand{\lcrx}[4][{-1}]{ 
	\IfEq{#1}{-1}{\left #2 {{{{#3}}}} \right #4}{
   	\IfEq{#1}{0}{#2 {{{{#3}}}} #4}{
	\IfEq{#1}{1}{\bigl #2 {{{{#3}}}} \bigr #4}{
	\IfEq{#1}{2}{\Bigl #2 {{{{#3}}}} \Bigr #4}{
	\IfEq{#1}{3}{\biggl #2 {{{{#3}}}} \biggr #4}{
	\IfEq{#1}{4}{\Biggl #2 {{{{#3}}}} \Biggr #4}{
    \GenericWarning{"4th argument to lcrx must be -1, 0, 1, 2, 3, or 4"}
    }}}}}}} %

\DeclareMathOperator*{\sgn}{sgn}

\newcommand{\setdelim}{\ : \ }

\newcommand{\condsym}{\ \vert \ }
\newcommand{\Bigcondsym}{\ \Big\vert \ }

\newcommand{\stT}{\ \text{\upshape s.t.}\ }

\newcommand{\ind}[1]{\II\{#1\}} %

\def\multiset#1#2{\ensuremath{\left(\kern-.3em\left(\genfrac{}{}{0pt}{}{#1}{#2}\right)\kern-.3em\right)}}

\DeclareMathOperator*{\argmin}{\arg\min} %
\DeclareMathOperator*{\argmax}{\arg\max} %
\DeclareMathOperator*{\newlim}{\mathrm{lim}\vphantom{\mathrm{infsup}}}
\DeclareMathOperator*{\newmin}{\mathrm{min}\vphantom{\mathrm{infsup}}}
\DeclareMathOperator*{\newmax}{\mathrm{max}\vphantom{\mathrm{infsup}}}
\DeclareMathOperator*{\newinf}{\mathrm{inf}\vphantom{\mathrm{infsup}}}
\DeclareMathOperator*{\newsup}{\mathrm{sup}\vphantom{\mathrm{infsup}}}
\renewcommand{\lim}{\newlim}
\renewcommand{\min}{\newmin}
\renewcommand{\max}{\newmax}
\renewcommand{\inf}{\newinf}
\renewcommand{\sup}{\newsup}

\newcommand{\dee}{\mathrm{d}} %
\newcommand{\grad}{\nabla} %

\newcommand{\bernoullidist}{\mathrm{Bernoulli}}

\newcommand{\normaldist}{\mathrm{Gaussian}}

\newcommand{\uniformdist}{\mathrm{Unif}}

\newcommand{\sbra}[2][{-1}]{\lcrx[#1] [ {#2} ] }

\newcommand{\abs}[2][{-1}]{\lcrx[#1] \vert {#2} \vert }

\newcommand{\norm}[2][{-1}]{\lcrx[#1] \Vert {#2} \Vert}
\newcommand{\inner}[3][{-1}]{\lcrx[#1] \langle {{#2},\ {#3}} \rangle}

\newcommand{\disteq}{\overset{d}{=}}

\newcommand{\Nats}{\NN}

\newcommand{\Reals}{\RR}

\newcommand{\PosReals}{\Reals_+}

\newcommand{\range}[2][{1}]{
	\IfEq{#1}{1}{\sbra{#2}}{\sbra{#2}_{#1}}}
\newcommand{\rangeO}[2][{0}]{
	\IfEq{#1}{0}{\sbra{#2}_0}{\sbra{#2}_{#1}}}

\definecolor{googleblue}{RGB}{66,133,244}
\definecolor{googlered}{RGB}{219,68,55}
\definecolor{googleyellow}{RGB}{244,180,0}
\definecolor{googlegreen}{RGB}{15,157,88}

\newcommand{\scarequo}[1]{``{#1}''}

\newcommand{\localname}{local}
\newcommand{\localName}{Local}

\newcommand{\covname}{feature}
\newcommand{\covName}{Feature}
\newcommand{\covnames}{features}
\newcommand{\covNames}{Features}

\newcommand{\intuitive}{(Informal)}

\newcommand{\obsname}{example}

\newcommand{\obsnames}{examples}

\newcommand{\radiusname}{radius}

\newcommand{\dataname}{response}

\newcommand{\algoname}{algorithm}

\newcommand{\queryalgoname}{query \algoname{}}

\newcommand{\modelname}{model}
\newcommand{\modelName}{Model}
\newcommand{\modelnames}{models}
\newcommand{\modelNames}{Models}

\newcommand{\behaviourname}{behaviour}
\newcommand{\behaviournames}{behaviours}
\newcommand{\behaviourName}{Behaviour}

\newcommand{\countmodbehaviourname}{counterfactual \modbehaviourname{}}

\newcommand{\countmodbehaviourNAME}{Counterfactual \modbehaviourNAME{}}

\newcommand{\modbehaviourname}{\modelname{} \behaviourname{}}

\newcommand{\modbehaviournames}{\modelname{} \behaviournames{}}
\newcommand{\modbehaviourNAME}{\modelName{} \behaviourName{}}

\newcommand{\methodname}{method}

\newcommand{\methodnames}{methods}
\newcommand{\methodNames}{Methods}

\newcommand{\queryname}{query}

\newcommand{\querynames}{queries}

\newcommand{\username}{user}

\newcommand{\usernames}{users}

\newcommand{\hyptestname}{hypothesis test}

\newcommand{\hyptestnames}{hypothesis tests}

\newcommand{\hyptestNAMEs}{Hypothesis Tests}

\newcommand{\oraclehyptestname}{\covname{}-\attname{} hypothesis test}

\newcommand{\oraclehyptestnames}{\covname{}-\attname{} hypothesis tests}

\newcommand{\oraclehyptestnaming}{\covname{}-\attname{} hypothesis testing}

\newcommand{\samphyptestname}{\queryname{} hypothesis test}

\newcommand{\samphyptestnaming}{\queryname{} hypothesis testing}

\newcommand{\nullname}{null hypothesis}

\newcommand{\bothhypnames}{null and alternate hypotheses}

\newcommand{\altname}{alternate hypothesis}

\newcommand{\sensname}{sensitivity}
\newcommand{\specname}{specificity}

\newcommand{\attname}{attribution}
\newcommand{\attName}{Attribution}
\newcommand{\attnames}{attributions}

\newcommand{\completename}{complete}
\newcommand{\completeName}{Complete}
\newcommand{\completenames}{completeness}
\newcommand{\completeNames}{Completeness}

\newcommand{\basename}{baseline}

\newcommand{\basenames}{baselines}

\newcommand{\linearname}{linear}

\newcommand{\additivename}{linear}
\newcommand{\additiveName}{Linear}
\newcommand{\additivenames}{linearity}
\newcommand{\additiveNames}{Linearity}

\newcommand{\fieldmethodname}{\covname{} \attname{} \methodname{}}

\newcommand{\fieldmethodnames}{\covname{} \attname{} \methodnames{}}

\newcommand{\defname}{definition}

\newcommand{\defnames}{definitions}

\newcommand{\pracname}{user}

\newcommand{\groundname}{ground truth}

\newcommand{\groundNAME}{Ground Truth}

\newcommand{\fielddefname}{\covname{} \attname{} \methodname{}}

\newcommand{\fielddefNAMEs}{\covName{} \attName{} \methodNames{}}
\newcommand{\fielddefnames}{\covname{} \attname{} \methodnames{}}

\newcommand{\taskname}{task}

\newcommand{\tasknames}{tasks}

\newcommand{\downtaskname}{end-task}

\newcommand{\downtasknames}{end-tasks}

\newcommand{\downtaskNAMEs}{End-Tasks}

\newcommand{\bothagnosticname}{\modelname{} agnostic}

\newcommand{\covspace}{\Xx}
\newcommand{\covnbhd}{\Bb}
\newcommand{\dataspace}{\Yy}
\newcommand{\modelspace}{\Ff}
\newcommand{\modelspacedum}{\Gg}

\newcommand{\queryspace}{\Qq}

\newcommand{\probspace}{\Pp}

\newcommand{\restrict}[1]{\lvert_{#1}\,}
\newcommand{\pwiselin}[1]{\Ll^{#1}}
\newcommand{\numpiece}{m}
\newcommand{\polytope}{V}

\newcommand{\jointdist}{\mu}
\newcommand{\covdist}{\jointdist}
\newcommand{\covdistdum}{\nu}

\newcommand{\covdim}{p}
\newcommand{\datadim}{q}

\newcommand{\threshval}{\alpha}

\newcommand{\recoursetask}{\mathrm{Recourse}}
\newcommand{\spurioustask}{\mathrm{Spurious}}

\newcommand{\linparamleft}{\beta^{\text{\tiny L}}}
\newcommand{\linparamright}{\beta^{\text{\tiny R}}}

\newcommand{\covobs}{X}

\newcommand{\covval}{x}

\newcommand{\covleft}{x^{\text{\tiny L}}}
\newcommand{\covright}{x^{\text{\tiny R}}}
\newcommand{\modelleft}{\model^{\text{\tiny L}}}
\newcommand{\modelright}{\model^{\text{\tiny R}}}
\newcommand{\covvaldum}{x'}

\newcommand{\covprob}{\tau}

\newcommand{\covrefval}{\overline\covval}

\newcommand{\dataval}{y}

\newcommand{\expval}{\phi}

\newcommand{\featuresub}{\Ss}

\newcommand{\shapwt}{\omega}

\newcommand{\feat}[1]{_{#1}}

\newcommand{\comp}[1]{{#1}^{\text{\tiny c}}}
\newcommand{\timeind}[1]{^{(#1)}}
\newcommand{\lipind}[1]{^{#1}}
\newcommand{\subspaceset}[1]{2^{#1}}

\newcommand{\leftind}{^{\text{\tiny L}}}
\newcommand{\rightind}{^{\text{\tiny R}}}

\newcommand{\defhelper}[1]{{\upshape #1}}
\newcommand{\lime}{\defhelper{LIME}}
\newcommand{\shap}{\defhelper{SHAP}}

\newcommand{\vanillagrad}{\defhelper{Gradient}}

\newcommand{\smoothgrad}{\defhelper{SmoothGrad}}
\newcommand{\smoothgradshort}{\defhelper{SG}}
\newcommand{\intgrad}{\defhelper{Integrated Gradients}}
\newcommand{\intgradshort}{\defhelper{IG}}

\newcommand{\model}{f}
\newcommand{\modeldum}{g}

\newcommand{\queryalgo}{\mathbf{q}}

\newcommand{\oraclehyptest}{\mathbf{h}}
\newcommand{\samphyptest}{\mathbf{s}}
\newcommand{\samphyptestoracle}{\samphyptest^*}
\newcommand{\samphyprv}[3]{\mathbf{T}\feat{\! #1,#2}\timeind{#3}}

\newcommand{\outscale}{\eps}
\newcommand{\inscale}{\delta}
\newcommand{\randsign}{\sigma}
\newcommand{\inoutratio}{r}
\newcommand{\lipconst}{L}
\newcommand{\lipdef}{\mathrm{Lip}}
\newcommand{\cube}{h}
\newcommand{\cubealt}{h^*}
\newcommand{\sampcompevent}{\Aa}

\newcommand{\covvalalt}{\covval^*}

\newcommand{\nullind}{\timeind{0}}
\newcommand{\altind}{\timeind{1}}
\newcommand{\nullhyp}{\mathrm{H}\nullind}
\newcommand{\althyp}{\mathrm{H}\altind}

\newcommand{\specsym}{\mathrm{Spec}}
\newcommand{\senssym}{\mathrm{Sens}}
\newcommand{\sampspecsym}{\mathrm{Spec}}
\newcommand{\sampsenssym}{\mathrm{Sens}}

\newcommand{\expdef}{\Phi}

\newcommand{\defhelp}[1]{^{\text{\tiny\upshape #1}}}
\newcommand{\limedef}{\expdef\defhelp{\lime{}}}

\newcommand{\intgraddef}{\expdef\defhelp{\intgradshort}}
\newcommand{\shapdef}{\expdef\defhelp{\shap}}

\newcommand{\smoothgraddef}{\expdef\defhelp{\smoothgradshort}}

\title{Impossibility Theorems for Feature Attribution}
\date{}
\author[1]{Blair Bilodeau\footnote{Work done during an internship at Google Brain. Correspondence to blair.bilodeau[at]gmail.com}}
\author[2,3]{Natasha Jaques}
\author[2,3]{Pang Wei Koh}
\author[3]{Been Kim}
\affil[1]{University of Toronto}
\affil[2]{University of Washington}
\affil[3]{Google Deepmind}

\setlength{\parindent}{0pt}
\setlength{\parskip}{5pt}
\usepackage{titlesec}

\titleformat{\paragraph}[runin]
  {\normalfont\normalsize\bfseries}
  {}
  {}
  {}
\titlespacing{\section}{0pt}{\parskip}{0pt}
\titlespacing{\subsection}{0pt}{\parskip}{0pt}
\titlespacing{\subsubsection}{0pt}{\parskip}{0pt}
\usepackage[letterpaper, margin=1in]{geometry}

\usepackage[numbers]{natbib}
\bibliographystyle{plainnat}

\let\oldparagraph=\paragraph
\renewcommand\paragraph[1]{\oldparagraph{#1.}}
}

\journal{

\templatetype{pnas-header-files/pnasresearcharticle}
\include{header-files/journal-global-macros}

\title{Impossibility Theorems for Feature Attribution}

\author[a,1,2]{Blair Bilodeau}
\author[b]{Natasha Jaques}
\author[b]{Pang Wei Koh}
\author[b]{Been Kim}

\affil[a]{University of Toronto}
\affil[b]{Google}

\leadauthor{Bilodeau}

\significancestatement{
Machine learning models can learn complex patterns from data,
but it is often difficult to understand why they make particular predictions.
To tackle this problem, practitioners typically turn to feature attribution methods, which seek to attribute the model's behaviour $f(x)$ around an example $x$ to particular features, or dimensions of $x$, that are most important for the prediction.
In recent years, a new class of feature attribution methods---namely, complete and linear methods---have become popular.
Our work shows that, unfortunately, such methods can be misleading: complete and linear methods are provably less reliable than simpler methods at answering basic feature attribution questions.
We provide impossibility results that highlight their failure cases and discuss how we might instead obtain reliable feature attributions.}

\authorcontributions{BB and BK conceived the project. BB and BK developed the theoretical results and formulated the experiments. BB ran the experiments. BB, NJ, PWK, and BK wrote the paper and discussed the overall project direction throughout.}
\authordeclaration{The authors declare no competing interests.}
\equalauthors{\textsuperscript{1}Work done during an internship at Google.}
\correspondingauthor{\textsuperscript{2}To whom correspondence should be addressed. E-mail: blair.bilodeau\@mail.utoronto.ca}


}

\begin{document}

\maketitle

\preprint{\vspace{-35pt}}

\begin{abstract}
Despite a sea of interpretability methods that can produce plausible explanations, the field has also empirically seen many failure cases of such methods.
In light of these results,
it remains unclear for practitioners how to use these methods and choose between them in a principled way.
In this paper,
we show that for moderately rich \modelname{} classes (easily satisfied by neural networks), any \fieldmethodname{} that is \completename{} and \additivename{}---for example, \intgrad{} and \shap{}---can provably fail to improve on random guessing for inferring \modbehaviourname{}.
Our results apply to common \downtasknames{} such as characterizing \localname{} \modbehaviourname{}, identifying spurious \covnames{}, and algorithmic recourse.
One takeaway from our work is the importance of concretely defining \downtasknames{}:
 once such an \downtaskname{} is defined,
a simple and direct approach of repeated \modelname{} evaluations can outperform
many other
complex \fieldmethodnames{}.

\end{abstract}

\journal{
\dates{This manuscript was compiled on \today}
\doi{\url{www.pnas.org/cgi/doi/10.1073/pnas.XXXXXXXXXX}}

\thispagestyle{firststyle}
\ifthenelse{\boolean{shortarticle}}{\ifthenelse{\boolean{singlecolumn}}{\abscontentformatted}{\abscontent}}{}
}

\preprint{
}

\preprint{\section{Introduction}\label{sec:intro}}

Feature attribution methods are commonly used to answer local counterfactual questions about machine learning \modelnames{}; that is, questions about a \modelname{}'s \behaviourname{} $\model(\covval)$ near a particular \obsname{} $\covval$.
For example, 
the goal of \emph{algorithmic recourse} is to determine what changes to $x$ a \pracname{} should make
to change the \modelname{}'s prediction,
such as whether raising an applicant's credit score would affect the \modelname{}'s predicted probability of loan default. 
Similarly, the goal of \emph{spurious \covname{} identification} is to determine whether \covnames{} that should be ignored by the \modelname{} actually affect $\model(\covval)$; for example, does the \modelname{} use a watermark on an X-ray to predict the probability of disease?
These \fieldmethodnames{} generally fall into one of two categories:
first, local approximations of $\model(\covval)$ in a sufficiently small neighbourhood around $\covval$, such as taking the gradient of $\model$ near $\covval$ \citep{simonyan13gradient} or using more sophisticated approximations like \smoothgrad{} \citep{smilkov17smoothgrad} and \lime{} \citep{ribeiro16lime};
and second, \methodnames{} that incorporate how the model behaves on a \emph{baseline distribution} $\covdist$ over examples that might be far from the example of interest, as in methods like \shap{} \citep{lundberg17shapley} and \intgrad{} \citep{sundararajan17integrated}.

The failure modes of the first category---local methods like taking the gradient of $\model$---are well understood. 
For example, if one is interested in recourse directions that extend outside of the neighbourhood used by a \localname{} \methodname{}, then the resulting \covname{} \attname{} may not reliably answer such questions.
These failures have spurred the development of methods in the second category, which are often motivated as provably satisfying properties such as \completenames{} and \additivenames{}. These methods---including \shap{} and \intgrad{} (\intgradshort{})---are commonly considered more reliable and are widely used to answer local counterfactual questions across many important applications,
including clinical trials \citep{liu21oncology},
medical alerts in the ICU \citep{zaeri18icus},
cancer diagnosis \citep{zhou21prostate} and prognosis \citep{roder21molecular}, and chemical binding \citep{mccloskey19attribution}.
As a concrete example, consider the clinical trial setting of \citet{liu21oncology}, where the \countmodbehaviourname{} of interest is whether removing certain eligibility criteria for participation in a clinical trial will change the efficacy of the trial (measured by the hazard ratio for patient survival).
The authors infer this \countmodbehaviourname{} by concluding that \scarequo{Shapley values close to zero [...] correspond to eligibility criteria that had no effect on the hazard ratio of the overall survival.}

In this work, we show that these common conceptions can be misleading: \completename{} and \additivename{} methods are in fact often less reliable than simpler methods---such as taking the gradient---at answering local counterfactual questions. 
Our main result is that for \emph{any} feature attribution method that satisfies the \completenames{} and \additivenames{} axioms, users cannot generally do better than random guessing for end-tasks such as algorithmic recourse and spurious feature identification. 
Specifically,
there are uncountably many pairs of \modelnames{} that share a \covname{} \attname{} yet have arbitrarily different \countmodbehaviourname{}.
Conversely, for every pair of distinct \covname{} \attname{} values, there are uncountably many pairs of \modelnames{} that match these \covname{} \attnames{} and yet have identical \countmodbehaviourname{}.
Furthermore, unlike simple local methods such as taking the gradient, \completename{} and \additivename{} methods remain unreliable even if we restrict our attention to infinitesimally small neighborhoods around the example of interest $\covval$, because they remain sensitive to how the model behaves on the baseline $\covdist$ that might be far from $\covval$.
Intuitively, there are two issues at play: (a) the \completenames{} axiom requires \attnames{} to sum to something meaningful, which is generally not well-aligned with answering questions about \countmodbehaviourname{}, and (b) the reliance on a \basename{} introduces additional failure modes due to how the model behaves far from the \obsname{} of interest.
Our results have direct implications for using \completename{} and \additivename{} \methodnames{} 
to answer questions about \countmodbehaviourname{}. 
For instance, positive \covname{} \attname{} does \textbf{not}, in general, imply that increasing the \covname{} will increase the \modelname{} output. 
Similarly, zero \covname{} \attname{} does \textbf{not}, in general, imply that the \modelname{} output is insensitive to changes in the \covname{} (see \cref{fig:intgrad-spurious-implications} for a visualization of false implications).
Beyond our general impossibility results,
these \methodnames{} can also fail to infer \countmodbehaviourname{} \emph{on average} over \modelnames{}. 
We show this both theoretically (over a distribution of polynomial models in \cref{sec:average-behaviour}) 
and empirically (on eight standard datasets, comprising tabular features and image classification,
in \cref{sec:experiments}).

In light of these results, we end by discussing how, 
in the absence of reliable \methodnames{}---that is, outside of special settings, such as when the model is linear or when the counterfactual is about an infinitesimally small change---%
we can reliably infer counterfactual model behaviour through a brute-force approach of querying the \modelname{} many times. 
We provide an example of such an approach for answering questions about spurious \covnames{}
and discuss how this insight can motivate future development of \fieldmethodnames{}.%

\begin{figure}[ht]
\centering
\includegraphics[width=0.75\linewidth]{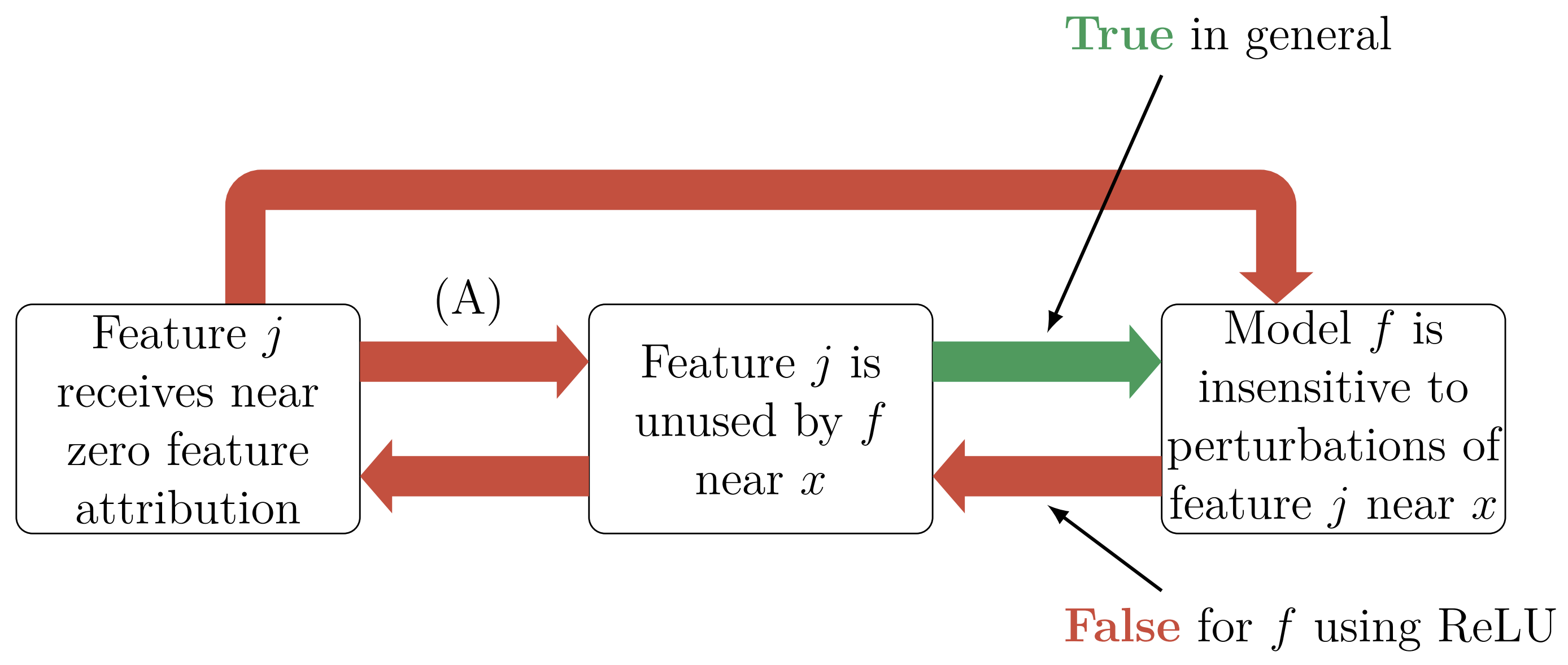}
\caption{
Red arrows indicate false implications for \completename{} and \additivename{} \fieldmethodnames{}, which follows from \cref{fact:oracle-hypothesis}.
Implication (A) is a standard belief in the literature for \fielddefnames{},
but we show it is false in general.
}
\label{fig:intgrad-spurious-implications}
\end{figure}

\section{Problem Framework}\label{sec:notation}

\subsection{The Topic of Interest: \countmodbehaviourNAME{}}\label{sec:counterfactual}
We start by defining \countmodbehaviourname{}, a simple yet general notion that is closely related to common \downtasknames{} such as recourse and spurious \covname{} detection.
Concretely, for a \modelname{} $\model$,
we consider a \username{} who wants to use the output of a \fielddefname{} to infer the \modelname{}'s dependence on \covname{} $j$ 
near an \obsname{} $\covval$ (that is, within a $\inscale$-neighbourhood of $\covval$).
More formally, for a fixed
\emph{\covname{} space} $\covspace \subseteq \Reals^\covdim$, \emph{\dataname{} space} $\dataspace \subseteq \Reals^\datadim$, and \emph{\modelname{}} $\model$ that is an element of a known \emph{\modelname{} class} $\modelspace\subseteq\covspace\to\dataspace$,
\countmodbehaviourname{} describes $\model(\covval')$ for $\covval'\in\covspace$ such that $\covval'\feat{j} \in (\covval\feat{j}-\inscale, \covval\feat{j}+\inscale)$ for a given radius $\inscale>0$ and \covname{} $j\in[\covdim]$.
Then, given a pair of candidate \emph{\modbehaviournames{}} $\modeldum\nullind, \modeldum\altind: (\covval\feat{j}-\inscale, \covval\feat{j}+\inscale) \to \dataspace$, the primary task we consider in this work is inferring which \modbehaviourname{} is more likely. 

Where do  $\modeldum\nullind$ and $\modeldum\altind$ come from? 
In general, the \countmodbehaviourname{} to be inferred is not precise enough to be described by a single $\modeldum$.
For instance, the \username{} may only want to infer whether $\model(\covval')$ increases as $\covval'\feat{j}$ increases, and clearly there are many possible $\modeldum\nullind$'s that satisfy this.
However, if the \username{} can reliably infer whether the \modelname{} is increasing, there must exist \emph{some} pair $\modeldum\nullind, \modeldum\altind$ (the former increasing, the latter decreasing) that the \username{} can reliably distinguish between.
Our results do not depend on knowing exactly what $\modeldum\nullind, \modeldum\altind$ is:
informally speaking, our theoretical results say that for \emph{every} such pair the \username{} \textbf{cannot} distinguish between them using the output of a \completename{} and \additivename{} \fielddefname{}, which then implies that they also cannot distinguish between the more general \behaviourname{} (such as increasing vs.\ decreasing).

How is \countmodbehaviourname{} related to \downtasknames{} that users may care about in real-world applications?
We consider two common applications in the literature, algorithmic recourse and spurious feature identification, to motivate inferring \countmodbehaviourname{}. 
For algorithmic recourse, we consider the task of determining whether increasing or decreasing a given \covname{} would be more beneficial.
For spurious feature identification, we consider the task of distinguishing if the \modelname{} output is sensitive to local perturbations in the \covname{}.
We formalize both of these tasks in \cref{sec:downtasks}.

\subsection{Framing the Problem: Hypothesis Testing}\label{sec:hypothesis-defs}

We formulate learning \countmodbehaviourname{} as a hypothesis testing problem, which gives us a framework to measure and compare the performance of different \fielddefnames{}.
Following the usual nomenclature, the \pracname{}'s goal is to determine whether the \modelname{} has certain \behaviourname{}: the \emph{\nullname{}}.
This \behaviourname{} is contrasted with some different, plausible \modbehaviourname{}: the \emph{\altname{}}.
The \nullname{} and \altname{} should encode necessary (but potentially not sufficient) questions that must be answered to succeed at the \taskname{} of inferring \countmodbehaviourname{}.
For example, for recourse the \username{} must be able to infer if the \modelname{} is increasing or decreasing, while for spurious \covnames{} the \username{} must be able to infer if the \modelname{} is sensitive or insensitive to perturbations of a \covname{}.
After collecting information about the \modelname{} using a \fielddefname{}, the \pracname{} will then conduct a hypothesis test and either \emph{reject} or \emph{fail to reject} the \nullname{}.

Formally, the \bothhypnames{} define subsets $\modelspace\nullind \subseteq \modelspace$ and $\modelspace\altind \subseteq \modelspace$, with
\*[
    &\nullhyp\setdelim \model \in \modelspace\nullind \\
    &\althyp\setdelim \model \in \modelspace\altind.
\]
Here, $\modelspace\nullind$ contains $\model\nullind$'s that locally agree with the possible $\modeldum\nullind$'s encoding \countmodbehaviourname{}, such as all \emph{increasing} \modelnames{}. Similarly, $\modelspace\altind$ contains all \modelnames{} with desired contrasting \modbehaviourname{}.
While it must be the case that these sets have no overlap ($\modelspace\nullind \cap \modelspace\altind = \emptyset$), it is often the case that they do not exhaustively contain all \modelnames{} ($\modelspace\nullind \cup \modelspace\altind \neq \modelspace$).
A \emph{\oraclehyptestname{}} is any way for the \pracname{}
to draw their conclusion for the \hyptestname{} solely on the output of a \fielddefname{} at an \obsname{}.
Formally, this is any function
\*[
    \oraclehyptest: \Reals^{\covdim\times\datadim} \to [0,1].
\]
The \username{} may rely on external randomness to decide whether to accept or reject the \nullname{}: the output of $\oraclehyptest$ is \emph{the probability that the \pracname{} rejects the \nullname{}}.

\subsection{Evaluating \hyptestNAMEs{}}

Most generally, 
a \emph{\fielddefname{}} is any function
\*[
    \expdef: \modelspace \times \probspace(\covspace) \times \covspace \to \Reals^{\covdim\times\datadim},
\]
where for any space $\Zz$, $\probspace(\Zz)$ denotes the set of all probability measures on $\Zz$.
This captures the most common properties of existing \defnames{} in the literature, facilitating their comparison. 
First, every \fielddefname{} must take as input the \modelname{} $\model\in\modelspace$ that the \pracname{} wishes to understand.
Second, the \pracname{} can specify an \obsname{} $\covval\in\covspace$ to \scarequo{localize} the \covname{} \attname{}.
In contrast to \scarequo{global} \methodnames{}, which seek to explain the effect of \covnames{} on the \modelname{} output across a population, localization allows the \pracname{} to seek 
\countmodbehaviourname{} relative to a specific \obsname{}.
Finally, many \methodnames{} rely heavily on the presence of a \emph{\basename{}}  $\covdist\in\probspace(\covspace)$ to incorporate \modbehaviourname{} from other \obsnames{} into computing \covname{} \attname{}.
This \basename{} may encode the training data distribution (e.g., this is what \shap{} does), concentrate entirely on a single \obsname{} (e.g., this is what \intgradshort{} does), or take some other form specific to the \taskname{} of interest. 

We concretely define a few methods widely used in practice and the focus of our analysis.
\vanillagrad{} \citep{simonyan13gradient} can be written as $\expdef(\model, \covdist, \covval) = \grad \model(\covval)$ (notice this is constant with respect to choice of $\covdist$).
Similarly, \intgradshort{} \citep{sundararajan17integrated} can be written as 
\*[
    \expdef(\model,\covdist,\covval)\feat{j} = \EE_{\covobs\sim\covdist}\Bigg[(\covval\feat{j} - \covobs\feat{j}) \int_0^1 \grad\feat{j}\model(\covobs + \alpha(\covval-\covobs)) \, \dee \alpha \Bigg],
\]
where $\covdist$ is often taken to be a pointmass.
For a more complete discussion of precise definitions and properties (including the definition of \shap{}), see \cref{sec:complete-additive-proofs}.

The goal of our work is to see if certain \fielddefname{}s can reliably be used to conduct the \hyptestnames{} described above. What are the right metrics to measure the success or failure?
Classically \citep{yerushalmy47hyptest},
the quality of a \hyptestname{} is determined by its \emph{\specname{}}, which is the probability that the \pracname{} fails to reject the \nullname{} whenever the \modelname{} satisfies the \nullname{} (a \emph{true negative}, or $1-$ probability of a type-II error), and its \emph{\sensname{}}\footnote{The interpretability literature sometimes uses the term \emph{sensitivity} to refer to the effect of perturbations on \covname{} \attnames{}  \citep{yeh19infidelity}. We use only the hypothesis testing definition above.}, which is the probability that the \pracname{} rejects the \nullname{} whenever the \modelname{} satisfies the \altname{} (a \emph{true positive}, or $1-$ probability of a type-I error).
In particular, for a fixed \fielddefname{} $\expdef$, \basename{} $\covdist\in\probspace(\covspace)$, \obsname{} $\covval\in\covspace$, and \bothhypnames{} $\modelspace\nullind,\modelspace\altind \subseteq \modelspace$, 
\*[
    \specsym_{\expdef,\covdist,\covval}(\oraclehyptest) &= 
    \inf_{\model\in\modelspace\nullind}
    [1-\oraclehyptest(\expdef(\model,\covdist,\covval))] \ \text{ and } \\
    \senssym_{\expdef,\covdist,\covval}(\oraclehyptest) &= 
    \inf_{\model\in\modelspace\altind}
    \oraclehyptest(\expdef(\model,\covdist,\covval)).
\]
For every \hyptestname{}, it is trivial to construct \emph{some} \modelname{} with accurate inference; our measure of performance instead asks if the \hyptestname{} does well for \emph{all} \modelnames{} of interest (mathematically, this is the role of the $\inf$\footnote{The reader may replace ``$\sup$" with ``$\max$" and ``$\inf$" with ``$\min$" without change of meaning; the difference is only that the spaces may not be compact and hence the $\argmax$ and $\argmin$ may not be achieved.}).
The goal of the \pracname{} is to simultaneously maximize \specname{} and \sensname{} (both take values in $[0,1]$). 

How does this framework differ from other hypothesis testing, such as t-testing?
Our definitions of \specname{} and \sensname{} are the exact analogues of the definitions used in classical statistical hypothesis testing. There, $\modelspace\nullind$ denotes the true data-generating \modelname{} rather than the learned \modelname{}, but the requirement that the test does well \emph{for all} such \modelnames{} is still enforced. As a concrete example, if the true data-generating model is assumed to be $\dataval = \model_\beta(\covval) = \beta\feat{1} \covval\feat{1} + \beta\feat{2} \covval\feat{2} + \eps$ where $\eps$ is standard Gaussian noise, then a canonical hypothesis test uses the \nullname{} $\modelspace\nullind = \{\model_\beta: \beta\feat{1} = 0\}$ to test if the first coefficient is non-zero. The standard requirement of using a \hyptestname{} $\oraclehyptest$ based on a test statistic with p-value less than $\alpha$ is exactly the requirement that
\*[
    \inf_{\model\in\modelspace\nullind} \EE_{\dataval_{1:n} \sim \model} [1 - \oraclehyptest(\dataval_{1:n})] \geq 1-\alpha.
\]
Since we have no randomness over the data (the \modelname{} is already learned, not retrained), our definition is a deterministic version of this standard worst-case \hyptestname{}.

\subsection{Notation}

We end this section with the standard notation used throughout the paper.

For any integer $i$, let $[i] = \{1,\dots,i\}$ and $e_i$ be the $i$th standard basis vector.
For any set $\Aa$, let $\comp{\Aa}$ denote its complement; also, let $\emptyset$ denote the empty set.
For any \obsname{} $\covval\in\covspace$, let $\covval\feat{j}$ denote the $j$th element/\covname{}. 

Often, we want to
evaluate $\model(z)$ where
$z\in\Reals^\covdim$ is defined such that
\*[
    z\feat{\ell}
    =
    \begin{cases}
        \covval\feat{\ell} &\ell \neq j\\
        \covvaldum\feat{\ell} &\ell = j
    \end{cases}
\]
for some \obsnames{} $\covval,\covvaldum\in\covspace$ and \covname{} $j\in[\covdim]$.
That is, replacing the $j$th \covname{} of $\covval$ with $\covvaldum\feat{j}$.
For clarity, we overload the notation of $\model$ and write $\model(\covval\feat{[\covdim]\setminus\{j\}}, \covvaldum\feat{j}) = \model(z)$. 
We use shorthand vector notation, letting $\covval\timeind{1:n} = (\covval\timeind{1},\cdots,\covval\timeind{n}) \in \covspace^n$. 
Similarly, we write $\model(\covval\timeind{1:n}) = (\model(\covval\timeind{1}),\dots,\model(\covval\timeind{n}))$.
For any matrix $M$, let $M_j$ denote the $j$th row of $M$.
We use capital letters to denote random variables, and 
for any $\covdist\in\probspace(\covspace)$ and $\model:\covspace\to\dataspace$, let 
\*[
    \EE_{\covobs\sim\covdist}[\model(\covobs)]
    = \int \model(\covval) \covdist(\dee \covval),
\]
with similar notation for variances, covariances, etc.
For any $j\in[\covdim]$, let $\covdist\feat{j}$ denote the $j$th marginal distribution of $\covdist$.
Note that, unless specifically mentioned, we make no assumption that the features are independent under $\covdist$.

\section{Impossibility Theorems}\label{sec:impossibility}

Our main result is that \textbf{for mildly rich \modelname{} classes, it is impossible to conclude that the \username{} does better than random guessing at inferring \countmodbehaviourname{} using common \fielddefnames{} without strong additional assumptions on the learning algorithm or data distribution.}
We also apply this result to two common \downtasknames{}, showing that under these conditions, it is impossible to conclude that the \username{} does better than random guessing for algorithmic recourse and spurious \covname{} identification using these \fielddefnames{}.
Finally, we show that even for very simple \modelnames{} (so our richness condition does not hold), these \fielddefnames{} provably still incorrectly infer \countmodbehaviourname{} with significant probability.

To characterize the common \fielddefnames{} we consider, we begin by defining two commonly used properties of such methods: \emph{\completenames{}} and \emph{\additivenames{}}.
Informally, \completenames{} requires that the individual \covname{} \attnames{} sum to the difference of the \modelname{} output from the \basename{}, while \additivenames{} requires that if the \modelname{} has no \covname{} interactions then the \attnames{} of a given \covname{} should be equivalent to considering that \covname{} on its own. 

\begin{definition}[\completeName{}]\label{defn:complete}
A \fielddefname{}
$\expdef$ is \emph{\completename{}} if and only if
for all
\modelnames{} $\model\in\modelspace$,
\basenames{} $\covdist\in\probspace(\covspace)$,
and
\obsnames{} $\covval\in\covspace$,
\*[
    \sum_{j\in[\covdim]}\expdef(\model,\covdist,\covval)\feat{j}
    = \model(\covval) - \EE_{\covobs\sim\covdist} \model(\covobs).
\]
\end{definition}

\begin{definition}[\additiveName{}]\label{defn:additive}
A \fielddefname{}
$\expdef$ is \emph{\additivename{}} if and only if
for all collections of functions $\model^{(1)},\dots,\model^{(\covdim)}:\Reals\to\dataspace$,
if $\model(\covval) = \sum_{j\in[\covdim]} \model^{(j)}(\covval\feat{j}) \in \modelspace$ then for all \basenames{} $\covdist\in\probspace(\covspace)$, \obsnames{} $\covval\in\covspace$, and \covnames{} $j\in[\covdim]$
\*[
    \expdef(\model,\covdist,\covval)\feat{j}
    = \expdef(\model\feat{j},\covdist\feat{j},\covval\feat{j}).
\]
\end{definition}

Two of the most common \fielddefnames{}, \shap{} \citep{lundberg17shapley} and \intgrad{} \citep{sundararajan17integrated}, are \completename{} and \additivename{} (for precise definitions and statements, see \cref{sec:complete-additive-proofs}).

\subsection{Assumptions on the \modelName{} }

Our theoretical results rely on two mild assumptions about the \covnames{} and the \modelname{} class.
These assumptions are typically satisfied in the cases we are interested in: neural networks applied beyond toy problems. %
This implies that to
do better than random guessing with common \fielddefnames{}, the \username{} must necessarily introduce more structure than this standard setting.

\paragraph{\cref{assn:inscale-covdist} \intuitive{}}
This assumption has two parts.
First, we require that $\covspace$ is not degenerate near $\covval$; i.e., there exists an interval around $\covval$ on which we can study \countmodbehaviourname{}.
If this were not true, it would not make sense to ask about the \modbehaviourname{} \scarequo{near $\covval$} (encoded by $\inscale$ as described in \cref{sec:counterfactual}).
In particular, we restrict ourselves to continuous input spaces. For sufficiently rich discrete input spaces (i.e., taking more than a few distinct values), our results could be proved in the same way with more notational overhead.
Second, we require that the \basename{} has support outside of this interval.
This assumption is mild in practice.
For example, when the \basename{} is the training distribution, there only needs to be a single training example that does not fall in the interval above.
When the \basename{} concentrates on a single \obsname{}, this \obsname{} must fall outside of the interval; in this case, the \basename{} \obsname{} usually corresponds to setting all features to zero, or all pixels to black, and hence is sufficiently far from any $\covval$ of interest. We formalize this as follows.

\begin{assumption}\label{assn:inscale-covdist}
For any \obsname{} $\covval\in\covspace$, \covname{} $j\in[\covdim]$, \radiusname{} $\inscale>0$, and \basename{} $\covdist\in\probspace(\covspace)$, we say that the present assumption holds if there exist $\covleft\feat{j},\covright\feat{j}\in\Reals$ such that
\*[\label{eqn:inscale-covdist}
    [\covval\feat{j}-\inscale,\covval\feat{j}+\inscale] \subseteq (\covleft\feat{j},\covright\feat{j}) \subseteq \{\covvaldum\feat{j}\setdelim\covvaldum\in\covspace\}
\]
and
\*[
    \covdist\feat{j}\Big((\covleft\feat{j}, \covright\feat{j}) \setminus [\covval\feat{j}-\inscale\feat{j}, \covval\feat{j}+\inscale\feat{j}]\Big) > 0.
\]
\end{assumption}

\paragraph{\cref{assn:piecewise} \intuitive{}}
General impossibility results are unobtainable; for example, most \fielddefnames{} correctly identify the \modelname{} for linear models. In other words, to derive an implication similar to ours (that the \covname{} \attname{} does not provide information about \countmodbehaviourname{}) one must impose some constraint on the \modelname{} class.
The purpose of this assumption is to define how rich a model would have to be for our result to hold. Note that this assumption
includes standard machine learning architectures such as neural networks with reasonable complexity.

In particular, we require that the \modelname{} class $\modelspace$ can represent sufficiently many piecewise linear extensions of the \localname{} \countmodbehaviourname{} (encoded by $\modeldum$ as described in \cref{sec:counterfactual}) that the \username{} wishes to infer.
This is naturally satisfied by ReLU networks: every ReLU network is a piecewise linear function, and any ReLU network of logarithmic (in dimension) depth or polynomial (in number of pieces) size can exactly replicate any piecewise linear function \citep[for precise bounds, see][]{arora18relu,tanielian21groupsort,chen22piecewise}.
This does not preclude the \modelname{} class from being arbitrarily richer than piecewise linear functions.

To precisely describe \cref{assn:piecewise}, we need some standard notation.
For any neighbourhood $\covnbhd\subseteq\covspace$ and \modelname{} $\model:\covspace\to\dataspace$, let $\model\restrict{\covnbhd}$ be the \modelname{} restricted to have the domain $\covnbhd$; that is, $\model\restrict{\covnbhd} = (\covnbhd\ni\covval\mapsto\model(\covval))$.
Similarly, let $\modelspace\restrict{\covnbhd} = \{\model\restrict{\covnbhd}\setdelim\model\in\modelspace\}$ denote the collection of all such restrictions.
We say that $\model$ is a \emph{$\numpiece$-piecewise linear function on $\covnbhd\subseteq\covspace$} if there exist closed, convex polytopes $\polytope\feat{1},\dots,\polytope\feat{\numpiece}$ that partition $\covnbhd$ such that $\model\restrict{\polytope\feat{j}}$ is linear for each $j\in[\numpiece]$ and $\model$ is continuous.
This is a standard definition, for example, see Definition~1 of \citet{chen22piecewise}.
Let $\pwiselin{\numpiece}\restrict{\covnbhd}$ denote the set of all $\numpiece$-piecewise linear elements of $\covnbhd\to\dataspace$.
We place the following assumption on the minimum richness of the \modelname{} class.

\begin{assumption}\label{assn:piecewise}
For any \obsname{} $\covval\in\covspace$, \covname{} $j\in[\covdim]$, \radiusname{} $\inscale>0$,
and \modbehaviourname{} $\modeldum:[\covval\feat{j}-\inscale, \covval\feat{j}+\inscale]\to\dataspace$,
we say that the present assumption holds if for the neighbourhood $\covnbhd = [\covval\feat{j}-\inscale, \covval\feat{j}+\inscale] \times \covspace\feat{[\covdim]\setminus\{j\}}$,
\*[
    \Big\{\model:
    \forall \covvaldum\in\covnbhd \ \model(\covvaldum) = \modeldum(\covvaldum\feat{j}), \ \model\restrict{\comp{\covnbhd}} \in \pwiselin{2}\restrict{\comp{\covnbhd}}\Big\} \subseteq \modelspace.
\]
\end{assumption}

If the \modbehaviourname{} $\modeldum$ is piecewise linear, then \cref{assn:piecewise} is satisfied by any $\modelspace$ expressive enough to realize piecewise linear functions.
If the \pracname{} is interested in \modbehaviourname{} that is more complex than piecewise linear, then it is reasonable to assume that $\modelspace$ is sufficiently expressive to represent this more complex \modbehaviourname{} locally, and thus \cref{assn:piecewise} is again satisfied if the \modelname{} class includes piecewise linear functions away from the region of interest.

\subsection{Main Result}

\begin{theorem}\label{fact:oracle-hypothesis}
Fix any \obsname{} $\covval\in\covspace$, \covname{} $j\in[\covdim]$,
\radiusname{} $\inscale>0$,
\basename{} $\covdist \in \probspace(\covspace)$,
and
\modbehaviourname{} $\modeldum\nullind,\modeldum\altind:[\covval\feat{j}-\inscale, \covval\feat{j}+\inscale]\to\Reals$.
Suppose that \cref{assn:inscale-covdist,assn:piecewise} are satisfied.
Let
\*[
    \modelspace\nullind
    &= \Big\{\model\in\modelspace\setdelim \forall\covvaldum\feat{j}\in[\covval\feat{j}-\inscale,\covval\feat{j}+\inscale], \ \model(\covval\feat{[\covdim]\setminus\{j\}},\covvaldum\feat{j}) = \modeldum\nullind(\covvaldum\feat{j}) \Big\} \\
    \modelspace\altind
    &= \Big\{\model\in\modelspace\setdelim \forall\covvaldum\feat{j}\in[\covval\feat{j}-\inscale,\covval\feat{j}+\inscale], \ \model(\covval\feat{[\covdim]\setminus\{j\}},\covvaldum\feat{j}) = \modeldum\altind(\covvaldum\feat{j}) \Big\}.
\]
For any \completename{} and \additivename{} \fielddefname{} $\expdef$ and $\oraclehyptest:\Reals^{\covdim\times\datadim}\to[0,1]$,
\*[
    \specsym_{\expdef,\covdist,\covval}(\oraclehyptest) \leq 1 -
    \senssym_{\expdef,\covdist,\covval}(\oraclehyptest).
\]
\end{theorem}

To interpret this result, consider that for any $\covprob\in[0,1]$, the trivial \hyptestname{} $\oraclehyptest \equiv \covprob$ that ignores the \modelname{}
will achieve
\*[
    \specsym_{\expdef,\covdist,\covval}(\oraclehyptest)
    = 1 - \senssym_{\expdef,\covdist,\covval}(\oraclehyptest).
\]
We refer to this family of trivial tests as \emph{random guessing}.
For example, one can always achieve $\specsym_{\expdef,\covdist,\covval}(\oraclehyptest) = 0$ at the expense of $\senssym_{\expdef,\covdist,\covval}(\oraclehyptest) = 1$ (by always rejecting the \nullname{}) or $\specsym_{\expdef,\covdist,\covval}(\oraclehyptest) = \senssym_{\expdef,\covdist,\covval}(\oraclehyptest) = 0.5$ (by ignoring the data and using a coin flip).
\cref{fact:oracle-hypothesis} shows that the best tradeoff between \sensname{} and \specname{} that can be achieved by \completename{} and \additivename{} \fielddefnames{} is no better than the tradeoff achieved by random guessing.
In other words, this result says that without imposing additional assumptions on the underlying data or learning algorithm to significantly reduce the \modelname{} complexity, the \username{} cannot conclude that they have learned any information about the \modelname{}.
In particular, they may do no better than random guessing at \downtasknames{} such as recourse and spurious \covname{} identification. %
In \cref{sec:experiments}, we demonstrate that this holds empirically for real data and real models.

For simplicity, we state our results for \countmodbehaviourname{} with respect to how the \modelname{} depends on a single \covname{} at a time, and defer the extension to arbitrary groups of \covnames{} as well as all proofs to \cref{sec:impossibility-proofs}.
The generality of \cref{fact:oracle-hypothesis} means that it applies to inferring any form of \countmodbehaviourname{}.
Returning to the clinical trial setting
of \citet{liu21oncology} in the introduction, the \username{} may wish to infer if the \modelname{} output changes as a function of a certain \covname{} (e.g., is the hazard ratio sensitive to the exclusion criteria).
To answer this, they must be able to distinguish between $\modeldum\nullind \equiv 0$ and some $\modeldum\altind$ that changes with \covname{} $j$.
However, \cref{fact:oracle-hypothesis} implies that for every such $\modeldum\altind$, the \username{} cannot conclude that they do better than random guessing at this inference task, and therefore, at identifying whether the \modelname{} depends on the \covname{}.
Thus, if this method is used for purposes such as determining membership in a clinical trial, the conclusions are not guaranteed to be any more reliable than random guessing. This has significant implications for using such methods in safety-critical or high stakes applications.

\subsection{Proof Sketch of \cref{fact:oracle-hypothesis}}\label{sec:proof-sketch}

The primary technical result (proved in \cref{sec:impossibility-proofs}) that facilitates our results is:

\begin{theorem}\label{fact:interior-freedom}
Fix any \obsname{} $\covval\in\covspace$, \covname{} $j\in[\covdim]$, 
\radiusname{} $\inscale>0$,
\basename{} $\covdist \in \probspace(\covspace)$, 
and
\modbehaviourname{} $\modeldum:[\covval\feat{j}-\inscale, \covval\feat{j}+\inscale]\to\dataspace$.
Suppose that \cref{assn:inscale-covdist,assn:piecewise} are satisfied. 
For every \attname{} $\expval \in \Reals^{\datadim}$, there exists 
a \modelname{} $\model\in\modelspace$ such that for every \completename{} and \additivename{} \fielddefname{}, $\expdef(\model,\covdist,\covval)\feat{j} = \expval$, and if $\abs[0]{\covval\feat{j}-\covvaldum\feat{j}} \leq \inscale\feat{j}$ then $\model(\covvaldum) = \modeldum(\covvaldum\feat{j})$.
\end{theorem}

Equipped with \cref{fact:interior-freedom}, we use it to prove \cref{fact:oracle-hypothesis} by constructing a counterexample. For any \nullname{} $\modelspace\nullind$ and \altname{} $\modelspace\altind$, we choose $\modeldum\nullind$ and $\modeldum\altind$ satisfying the \countmodbehaviourname{} prescribed by the \nullname{} and \altname{} respectively. For example, in the recourse setting, $\modeldum\nullind$ may be increasing in the \covname{} while $\modeldum\altind$ is decreasing.
Then, by \cref{fact:interior-freedom}, we can construct $\model\nullind$ and $\model\altind$ that locally agree with $\modeldum\nullind$ and $\modeldum\altind$ respectively, yet receive the same \covname{} \attname{}. Consequently, any \oraclehyptestname{} that relies on only the \covname{} \attname{} to draw conclusions cannot distinguish between $\model\nullind$ and $\model\altind$, and hence has been provided no additional information to do better than random guessing.
While a single counterexample is sufficient to prove \cref{fact:oracle-hypothesis}, the proof of \cref{fact:interior-freedom} reveals that we can actually find uncountably many \modelnames{} that work as counterexamples, which has important implications for practical settings.
We demonstrate this in our on-average results (\cref{sec:average-behaviour}) and experiments (\cref{sec:experiments}).

We visualize the intuition behind this result in \cref{fig:attribution}, where we show (a) \modelnames{} can behave very differently and all receive the same attribution and (b) \modelnames{} can be identical in a local neighbourhood of interest yet receive very different attribution. 
Critically, our counterexamples are very simple (piecewise linear in the proof, piecewise quadratic for smooth visualization in \cref{fig:attribution}), and hence easily realized by neural networks.
Finally, since \fielddefnames{} provide only a summary of the \modelname{}, it is clear that one can always find \emph{some} example of two different \modelnames{} that receive the same \attname{}.
The primary insight that enables our results is that we can always find \modelnames{} that \emph{differ in exactly the way that we wish to test} yet receive the same \attname{}. The next natural question is: how does this result extend to \downtasknames{} that practitioners care about?

\begin{figure}
\centering
\includegraphics[scale=0.6]{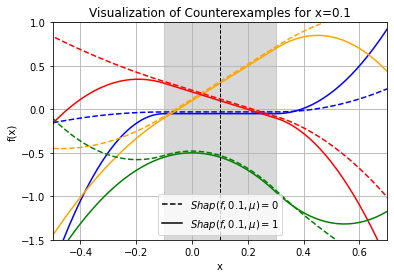}
    \caption{Each line represents a different one-dimensional \modelname{}. For $\covval=0.1$ and $\covdist=\uniformdist(-1,1)$, dashed lines receive \shap{}$(\model,\covval,\covdist)=0$ while solid lines receive \shap{}$(\model,\covval,\covdist)=1$. The behaviour of models with the same colour is identical within the shaded region, which denotes the neighbourhood $(\covval-\inscale,\covval+\inscale)$ for $\inscale=0.2$. \modelNames{} can behave very differently and all receive the same \attname{} (e.g., all dashed lines) and \modelnames{} can be identical in a neighbourhood yet receive very different \attname{} within that neighbourhood (e.g., lines with the same colour). }
    \label{fig:attribution}
\end{figure}

\subsection{Results for \downtaskNAMEs{} of Interest}

We discuss three \downtasknames{}, starting with the simplest one (\localname{} \modbehaviourname{}) which forms the basis of the two more common \downtasknames{}: recourse and spurious correlation.
There are multiple notions of these \tasknames{} in the literature, and they are often used without a precise formalization.
To facilitate analysis, we provide \emph{one possible} formalization of the easiest version of these \tasknames{}.
We defer general definitions and full proofs to \cref{sec:application-proofs}.

\subsubsection{Simple \localName{} \modbehaviourNAME{}}\label{sec:local-behaviour}

A consequence of \cref{fact:oracle-hypothesis} is that if the notion of \countmodbehaviourname{} is sufficiently local,
\shap{} and \intgradshort{} may not provide valid inference for $\model$. Specifically, we consider the task of identifying whether a \modelname{} output is locally sensitive to perturbations of a \covname{}, which is a necessary component of identifying spurious features: if one can reliably identify spurious features, one must be able to identify \covnames{} for the which the \modelname{} output is insensitive to small perturbations. We first show that \vanillagrad{} is \emph{always} successful at this \taskname{} for a sufficiently small notion of perturbations.
While the limitations of \vanillagrad{} are well-known for many \tasknames{} \citep{adebayo18sanity} and there are additional challenges for \vanillagrad{} computed using backpropagation \citep{nie18backprop}, the following result demonstrates that there exists a \taskname{} for which \vanillagrad{} is reliable yet \shap{} and \intgradshort{} may not be.

\begin{proposition}\label{fact:local-stability-grad}
Fix $\covval\in\covspace$ and suppose $\modelspace \subseteq \{\model: \Reals^\covdim \to \Reals, \grad\model(\covval) \ \mathrm{ exists }\}$.
For every $\outscale>0$, $\covval\in\covspace$, and $j\in[\covdim]$, there exists $\inscale>0$ such that if
\*[
    \modelspace\nullind
    &= \Big\{\model\in\modelspace: \sup_{\alpha\leq\inscale} \abs{\model(\covval+\alpha e_j) - \model(\covval)} \leq \inscale \outscale / 2 \Big\} \\
    \modelspace\altind
    &= \Big\{\model\in\modelspace: \sup_{\alpha\leq\inscale} \abs{\model(\covval+\alpha e_j) - \model(\covval)} > \inscale \outscale \Big\},
\]
then there exists $\oraclehyptest$ using $\expdef=$ \vanillagrad{} such that for every $\covdist\in\probspace(\covspace)$
\*[
    \specsym_{\expdef,\covdist,\covval}(\oraclehyptest)
    = \senssym_{\expdef,\covdist,\covval}(\oraclehyptest)
    = 1.
\]
\end{proposition}

Using \cref{fact:oracle-hypothesis}, we next show that \completename{} and \additivename{} \fieldmethodnames{} provably are less reliable than \vanillagrad{} for this task.

\begin{proposition}\label{fact:local-instability}
Fix $\covval\in\covspace$ and suppose $\{\model: \Reals^\covdim \to \Reals, \model \ \text{\upshape is piecewise linear and }  \grad\model(\covval) \ \mathrm{ exists }\} \subseteq \modelspace \subseteq \{\model: \Reals^\covdim \to \Reals, \grad\model(\covval) \ \mathrm{ exists }\}$.
For every sufficiently small $\outscale,\inscale>0$, $\covval\in\covspace$, and $j\in[\covdim]$, if
\*[
    \modelspace\nullind
    &= \Big\{\model\in\modelspace: \sup_{\alpha\leq\inscale} \abs{\model(\covval+\alpha e_j) - \model(\covval)} \leq \inscale \outscale/2 \Big\} \\
    \modelspace\altind
    &= \Big\{\model\in\modelspace: \sup_{\alpha\leq\inscale} \abs{\model(\covval+\alpha e_j) - \model(\covval)} > \inscale \outscale \Big\},
\]
then for every \oraclehyptestname{} $\oraclehyptest$, \completename{} and \additivename{} \fielddefname{} $\expdef$, and $\covdist\in\probspace(\covspace)$ satisfying \cref{assn:inscale-covdist},
\*[
    \specsym_{\expdef,\covdist,\covval}(\oraclehyptest)
    \leq 1 - \senssym_{\expdef,\covdist,\covval}(\oraclehyptest).
\]
\end{proposition}

\subsubsection{Recourse and Spurious \covNames{}}\label{sec:downtasks}

While \cref{sec:local-behaviour} proves that \completename{} and \additivename{} \fielddefnames{} can be unreliable for inferring sufficiently local \countmodbehaviourname{}, common \downtasknames{} in the literature are only \scarequo{moderately} local.
We now show that inference for these \tasknames{} can also be as uninformative as random guessing.
To do so,
we formalize increasing and decreasing a \covname{} by considering a distribution over perturbations and measuring the \modelname{} change on average with respect to this distribution (which may differ from the \basename{} used to compute the \covname{} \attname{}).

\begin{definition}[Recourse]\label{defn:recourse}
Fix any $\covval\in\covspace$, $j\in[\covdim]$, $k\in[\datadim]$,
$\inscale >0$, and
$\covdistdum \in \probspace(\covspace)$.
Let
\*[
    \modelspace\nullind
    &= \Big\{
    \model\in\modelspace\setdelim \EE_{\covobs\sim\covdistdum}\Big[\model(\covval\feat{[\covdim]\setminus\{j\}}, \covobs\feat{j})\feat{k} \condsym \covobs\feat{j} \in [\covval\feat{j}, \covval\feat{j}+\inscale]\Big] \\
    &\hspace{16mm}> \EE_{\covobs\sim\covdistdum}\Big[\model(\covval\feat{[\covdim]\setminus\{j\}}, \covobs\feat{j})\feat{k} \condsym \covobs\feat{j} \in[\covval\feat{j}-\inscale,\covval\feat{j}]\Big]\Big\} \\
    \modelspace\altind
    &= \modelspace \setminus \modelspace\nullind.
\]
\end{definition}

\paragraph{Recourse \intuitive{}}
Here, $\covdistdum$ is a distribution from which perturbed \obsnames{} can be sampled from.
This distribution need not be the same as the \basename{} used by the \fielddefname{}---a common choice would be the uniform distribution or a pointmass for a single perturbation of interest.
Then, $\modelspace\nullind$ corresponds to those \modelnames{} where increasing \covname{} $j$ will increase the \modelname{} output on average.
Again as a concrete example, if the \modelname{} output is the probability of loan acceptance and \covname{} $j$ is income, the \pracname{} asks whether to increase or decrease credit score in order to improve the probability of acceptance.

Next, we formalize distinguishing whether a \modelname{} is sensitive or insensitive to perturbations.

\begin{definition}[Spurious \covNames{}]\label{defn:spurious}
Fix any $\covval\in\covspace$, $j\in[\covdim]$, $k\in[\datadim]$,
$\inscale >0$, and $\outscale>0$.
Let
\*[
    \modelspace\nullind &= \Big\{\model\in\modelspace \setdelim \sup_{\covvaldum\feat{j}\in(\covval\feat{j},\covval\feat{j}+\inscale]} \abs[0]{\model(\covval\feat{[\covdim]\setminus\{j\}},\covvaldum\feat{j})\feat{k}} \, = 0\Big\} \\
    \modelspace\altind &= \Big\{\model\in\modelspace \setdelim \sup_{\covvaldum\feat{j}\in(\covval\feat{j},\covval\feat{j}+\inscale]} \abs[0]{\model(\covval\feat{[\covdim]\setminus\{j\}},\covvaldum\feat{j})\feat{k}} \, \geq \outscale\Big\}.
\]
\end{definition}

\paragraph{Spurious \covNames{} \intuitive{}}
$\modelspace\nullind$ corresponds to those \modelnames{} where the \modelname{} output does not change from perturbing \covname{} $j$.
Meanwhile, $\modelspace\altind$ corresponds to the \modelnames{} that have a \scarequo{significant} change from perturbing \covname{} $j$, where
significance is encoded by the size of $\outscale$.
Concretely, if \covnames{} correspond to the pixels of a watermark on an X-ray, the \pracname{} asks whether the \modelname{} output is sensitive to the values of these pixels.

We now apply \cref{fact:oracle-hypothesis} to these definitions and obtain the following implications.

\begin{corollary}\label{fact:oracle-hypothesis-examples}
Fix any $\covval\in\covspace$, $j \in [\covdim]$,
$\inscale>0$,
and $\covdist\in\probspace(\covspace)$ such that \cref{assn:inscale-covdist} is satisfied.
Fix
$k\in[\datadim]$ and
$\covdistdum \in \probspace(\covspace)$,
and
let $\modelspace\nullind$ and $\modelspace\altind$ be as defined in either \cref{defn:recourse} or \cref{defn:spurious}.
Suppose that there exists $\model\nullind\in\modelspace\nullind$ and $\model\altind\in\modelspace\altind$ such that $\modeldum\nullind = \model\nullind\restrict{\covnbhd}$ and $\modeldum\altind = \model\altind\restrict{\covnbhd}$ each satisfy \cref{assn:piecewise}.
Then
for any \completename{} and \additivename{} \fielddefname{} $\expdef$ and \oraclehyptestname{} $\oraclehyptest$,
\*[
    \specsym_{\expdef,\covdist,\covval}(\oraclehyptest) \leq 1 -
    \senssym_{\expdef,\covdist,\covval}(\oraclehyptest).
\]
\end{corollary}

As previously mentioned, \cref{fact:oracle-hypothesis-examples} implies that the \pracname{} cannot distinguish whether increasing or decreasing the \covname{} is the correct direction to increase the \modelname{} prediction.
For recourse, the main assumption that \cref{assn:piecewise} holds is satisfied in this case when $\modelspace$ contains piecewise linear functions, since the functions $\modeldum(\covval\feat{j}) = \covval\feat{j}$ for $\model\nullind$ and
$\modeldum(\covval\feat{j}) = -\covval\feat{j}$ for $\model\altind$ suffice.
Similarly, \cref{fact:oracle-hypothesis-examples} implies that the \pracname{} cannot distinguish whether the \modelname{} prediction is sensitive to changes in the \covname{}.
For spurious features, the main assumption that \cref{assn:piecewise} holds is again satisfied when $\modelspace$ contains piecewise linear functions, since the functions $\modeldum(\covval\feat{j}) = 0$ for $\model\nullind$ and
$\modeldum(\covval\feat{j}) = \outscale$ for $\model\altind$ suffice.

\subsection{Average \localName{} Performance for Simple \modelNames{}}\label{sec:average-behaviour}

As motivated in \cref{sec:notation}, the usual definitions of \specname{} and \sensname{} demand accurate \hyptestnames{} for every \modelname{} of interest. We now show that for standard distributions placed over simple \modelname{} classes, our results apply even for the \emph{easier} task of obtaining accurate \hyptestnames{} on average over the models.

Consider the simple class of univariate \modelnames{} $\modelspace = \{\covval\mapsto a \covval^n - \covval \setdelim a\in\Reals\}$ for some fixed $n \geq 2$. Although this \modelname{} class is so simple that \cref{assn:piecewise} does not apply (and hence we cannot directly apply \cref{fact:oracle-hypothesis}), we can prove a similar result that applies on average.

\begin{proposition}\label{fact:random-quadratic}
Let $\pi$ denote the distribution over $\modelspace$ induced by $a \sim \normaldist(0,1)$, let $\covdist$ be such that $\EE_\covdist \covobs^n \in (1/2,1)$, and set $\covval=1$. Then, for any \completename{} and \additivename{} $\expdef$,
\*[
    \PP_{\model\sim\pi}\Big[\sgn(\expdef(\model,\covdist,\covval)) \neq \sgn(\model'(\covval)) \Big] \in (0.25, 0.5).
\]
\end{proposition}

As a consequence, any \oraclehyptestname{} that relies on the sign of a \completename{} and \additivename{} \fielddefname{} to infer \localname{} \countmodbehaviourname{} will draw the wrong conclusion at least 1/4 of the time, even for this exceedingly simple model class.
In the next section, we show empirically that our results apply on average in the more complex settings for which \cref{fact:oracle-hypothesis} applies.

\section{Experiments}\label{sec:experiments}

The theoretical guarantees of the previous section demonstrate that common \fielddefnames{} such as \shap{} and \intgradshort{} cannot reliably infer \countmodbehaviourname{}. 
While our assumptions are satisfied by moderately rich \modelname{} classes (including neural networks), 
our theory does not rule out the possibility that additional structure extracted from the training data or learning algorithm %
is further aiding \fielddefnames{} in practice.
In this section, we provide experimental results consistent with what our theory predicts, even when we restrict consideration to neural networks trained with stochastic gradient descent on real datasets.
Specifically, for ReLU neural networks on tabular data and convolutional neural networks on image data, we observe that \shap{} and \intgradshort{} are close to random guessing for the \downtasknames{} of algorithmic recourse and spurious \covname{} identification.
We also compare with three common local methods: gradients, \smoothgrad{}, and \lime{}. 
We find that for simple tabular data, these methods can outperform \shap{} and \intgradshort{}, while for image data all methods are comparable to random guessing.
All code is available at \url{https://github.com/google-research/interpretability-theory}.

\subsection{Methods}

To visualize \specname{} and \sensname{}, we use the standard receiver operating characteristic (ROC) curve, which shows the trade-off of the false positive rate (1 -- \specname{}) on the $x$-axis and the true positive rate (\sensname{}) on the $y$-axis. An ROC curve is computed by varying the rejection threshold for a \hyptestname{}: the more strict the threshold, the less likely to reject the null, and hence the lower the false positive rate. An ideal hypothesis test threshold achieves the top left corner of the plot (0\% false positives and 100\% true positives), and generally a hypothesis test is better if the curve is closer to the top left corner of the plot. The diagonal line from $(0,0)$ to $(1,1)$ is the line \specname{} = 1 -- \sensname{}, and hence corresponds exactly to random guessing.

For each dataset, \fielddefname{}, and \downtaskname{}, we construct an empirical ROC curve.
To do so, for each dataset we retrain 10 neural networks (\modelnames{}) using different random seeds to similar accuracy. Then, we randomly sample 20 \obsnames{} from the test dataset and compute each \fielddefname{} on each \obsname{} for each \modelname{}. For each non-categorical \covname{} and each \downtaskname{}, we compute a \hyptestname{} at a specific threshold using the \covname{} \attname{} and compare it to a ``\groundname{}''; for details on the \hyptestnames{} and the \groundname{}, see \cref{sec:examples-tasks}.
Thus, each point on a plot corresponds to a single dataset, \modelname{}, \fielddefname{}, \downtaskname{}, and \hyptestname{} threshold, and represents the empirical true and false positive rates for a \hyptestname{} averaged over all \covnames{} on 20 \obsnames{}.
We repeat each of these calculations with 40 different \hyptestname{} thresholds to create the entire ROC curve.
Since we reuse the same 20 \obsnames{} for each \modelname{}, the noisy ROC curve is actually comprised of 10 different (one for each \modelname{}) monotonic ROC curves.

For image data, it is ambiguous what constitutes a ``\covname{}''. Following the literature \citep{molnar2022}, we consider each individual pixel to be a \covname{}. For computational reasons, we did not average over every pixel for each \hyptestname{}, but instead over a sample of 10 pixels (this matches that most of the tabular datasets have roughly 10 \covnames{}). 

\subsubsection{\downtaskNAMEs{}}\label{sec:examples-tasks}
\paragraph{\groundNAME{}}
For both \downtasknames{}, the \groundname{} relies on a neighbourhood around \obsnames{}. For each \obsname{} $\covval$ and \covname{} $j$ we consider, we construct this neighbourhood as follows. We fix a percentage $p \in [0,1]$, and compute a range $R$ which is the maximum value of the $j$th \covname{} minus the minimum value of the $j$th \covname{} on the dataset under consideration. We then create 20 copies of $\covval$, where all \covname{} values are fixed except for the $j$th, which is set to $\covvaldum_j = \covval_j + \delta$ for $\delta$ evenly spaced in $(-pR, pR)$. We display results for $p=0.1$, but found similar results for $p$ ranging from $0.5$ to $0.01$. For smaller $p$, we encountered floating point issues (the model output didn't change at all over the neighbourhood).

\textbf{Recourse.} 
We use the \taskname{} of %
\cref{defn:recourse} for a single \covname{}, with $\covdistdum$ taken to be the uniform distribution.
To compute the \groundname{}, we approximate expectation under this distribution by comparing the empirical average \modelname{} output on the first half of the perturbed \obsnames{} to the empirical average \modelname{} output on the second half of the perturbed \obsnames{}.
That is, the \groundname{} is 1 if the average \modelname{} output is larger from increasing the \covname{} versus decreasing it, and 0 otherwise.

To conduct the \hyptestname{}, we use the sign and magnitude of the \covname{} \attname{}.
In particular, for threshold $\threshval\in\Reals$ and \covname{} \attname{} $\expval\in\Reals$, we use $\oraclehyptest(\expval) = \ind{\expval > \threshval}$.
This hypothesis test is consistent with applications in the literature: \citet{aditya20biased} use \shap{} (also accounting for magnitude) and \citet{ghosh22fair} use \intgradshort{} to identify which \covnames{} should be adjusted for to achieve outcome fairness.

\textbf{Spurious \covNames{}.} 
We use a variant of the  \taskname{} of \cref{defn:spurious}, replacing $\sup$ with variance for better stability.
First, we compute the perturbation described above for 100 additional \obsnames{} from the dataset of interest and then compute the variance of the \modelname{} output over each perturbation (providing an empirical distribution of such variances). 
We then set the $\outscale$ in \cref{defn:spurious} to be the 80th quantile of this empirical distribution, and the \groundname{} is computed by comparing the \modelname{} output variance over the perturbations of the \obsname{} at hand with this $\outscale$. 
That is, the \groundname{} is 1 if perturbing the \covname{} causes the \modelname{} output to vary more than 80\% of other \covnames{}, and 0 otherwise.

To conduct the \hyptestname{}, we use only the magnitude of the \covname{} \attname{}.
In particular, for threshold $\threshval\in\Reals$ and \covname{} \attname{} $\expval\in\Reals$, we use $\oraclehyptest(\expval) = \ind{\abs{\expval} > \threshval}$.

\subsubsection{\fielddefNAMEs{}}

\textbf{\shap{}.} For tabular data (i.e., a small number of \covnames{}), we compute \shap{} according to \cref{def:genshap} using the \texttt{KernelExplainer} function from the Python \texttt{shap} package \citep{lundberg17shapley}. We approximate the outer expectation (with respect to the training data distribution) using an empirical average over 100 samples and the inner expectation (with respect to the Shapley kernel over subsets) using an empirical average over 500 samples. For image data, the number of subsets of pixels is too large to approximate accurately, so we follow the \texttt{shap} documentation and use the \texttt{Explainer} function with \texttt{maskers.Image}. This approximates the \shap{} value by ``blocking out'' groups of pixels.

\textbf{\vanillagrad{}.} We compute $\grad \model$ exactly using \texttt{TensorFlow} \citep{tensorflow2015-whitepaper}.

\textbf{\intgrad{}.} We use \cref{def:intgrad}, approximating the integral using a sum with 20 steps. Following \citet{sundararajan17integrated}, we use a pointmass \basename{}, so the expectation does not need to be approximated. 
We visualize two baselines: all \covnames{} set to zero (the mean of the data after rescaling), and all \covnames{} set to their minimum value.

\textbf{\smoothgrad{}.} We use \cref{def:smoothgrad}, where for each $\covval$ we take $\covdist$ to be $\normaldist(\covval, 0.1 \cdot I_\covdim)$ following \citet{smilkov17smoothgrad}. We approximate this expectation using 100 samples.

\textbf{\lime{}.} We use \cref{def:lime}, following the use of the best (regularized) local linear model from \citet{ribeiro16lime}. For simplicity, we set $\lambda = 1$ and for each $\covval$ take $\covdist$ to be $\normaldist(\covval, 0.1 \cdot I_\covdim)$---in \citet{ribeiro16lime}, $\covdist$ is denoted by $\pi_\covval$. We then approximate the expectation using 100 samples so that the $\argmin$ can be found exactly using the closed-form solution of regularized least squares.

\subsubsection{Datasets}

\textbf{Tabular data.} We consider 5 standard tabular datasets from the UCI repository \citep{Dua19UCI}: wine origin (\texttt{wine}) \citep{forina86wine}, credit approval  (\texttt{credit}) \citep{quinlan1987credit}, chess outcome (\texttt{chess}) \citep{bain1994chess}, E.~coli protein localization (\texttt{ecoli}) \citep{nakai1991ecoli}, and abalone age (\texttt{abalone}) \citep{nash1994abalone}. After centering and normalizing the \covnames{} by their empirical standard deviation, we trained small ReLU neural networks on the first four datasets to average test accuracies of 100\%, 80\%, 99\%, and 85\% respectively, while for the \texttt{abalone} dataset (which has integer responses) we achieved average test mean squared error of 4.5.

\textbf{Image data.} We consider 3 standard image datasets: MNIST digit classification (\texttt{mnist}) \citep{lecun2010mnist}, Fashion-MNIST classification (\texttt{fashion}) \citep{xiao17fashion}, and CIFAR-10 image classification (\texttt{cifar-10}) \citep{Krizhevsky09learningmultiple}.
After normalizing the pixel values to $[0,1]$, we trained standard convolutional neural networks to average test accuracies of 99\%, 90\%, and 80\% respectively.

\subsection{Results}

We plot ROC curves for all \downtasknames{}, datasets, and attribution methods in \cref{fig:10percent-plots-1} (tabular datasets) and \cref{fig:10percent-plots-2} (image datasets).
We highlight some observations:

\textbf{Observation 1:}
\shap{} and \intgrad{} ROC curves are near random guessing for almost all experiments. This agrees with what our theory suggests.

\textbf{Observation 2:}
The \basename{} for \intgradshort{} matters.
Using a \basename{} of all zeroes is, in general, worse than using a \basename{} corresponding to the minimum \covnames{}.
This difference is to be expected: the choice of \basename{} is already empirically known to heavily influence the output \citep{sturmfels2020visualizing}.
However, we emphasize that (a) both \basenames{} are often near the random guessing line, and (b) \emph{without strong additional knowledge about the dataset and learned \modelname{}, there is no way to know in advance what the ``right'' \basename{} is to choose for your \taskname{}}.

\textbf{Observation 3:}
Simpler, local \fielddefnames{} sometimes outperform \shap{} and \intgradshort{}.  
For some tasks, gradients, \smoothgrad{}, and \lime{} perform much better than random guessing, likely because algorithmic recourse and spurious \covname{} identification are local \downtasknames{}.
These methods are not foolproof, however; for many cases they also fail to improve on random guessing.
Once again, \emph{there is no way to know in advance whether your \fielddefname{} will work for your \taskname{} without strong additional knowledge}.

\textbf{Observation 4:}
The ROC curves suggest that, in general, the end-tasks are easier to solve on tabular datasets than image datasets. 
While for tabular datasets we see many ROC curves far from random guessing, the ROC curves for image datasets are near the diagonal that corresponds to random guessing. 
We conjecture that this is because the \modelnames{} learned for tabular datasets can be quite simple, making gradients (and their variants) more indicative of \countmodbehaviourname{}. (Recall that if the \modelname{} is linear, \countmodbehaviourname{} is recovered exactly by all \methodnames{}.)

\clearpage
\begin{figure*}[!hpt]
\centering
\begin{subfigure}
    \centering
    \includegraphics[width=.33\linewidth]{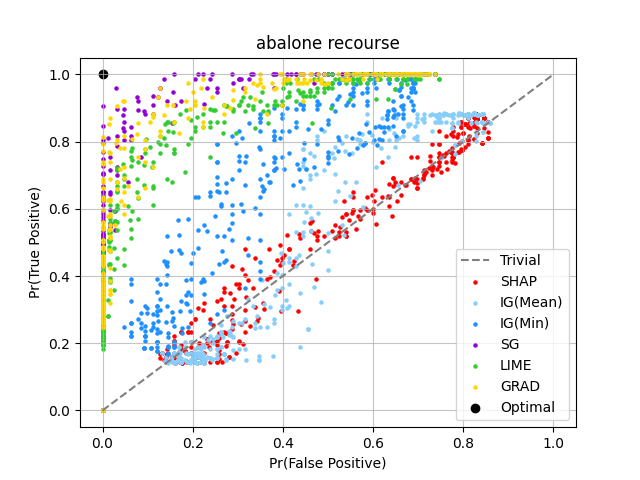}
\end{subfigure}
\begin{subfigure}
    \centering
    \includegraphics[width=.33\linewidth]{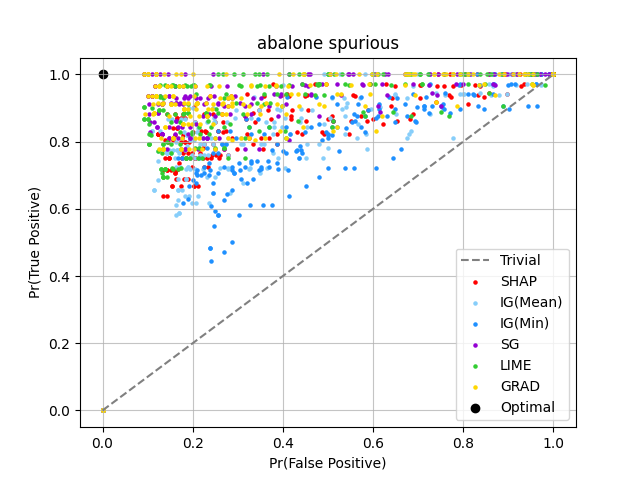}
\end{subfigure}
\begin{subfigure}
    \centering
    \includegraphics[width=.33\linewidth]{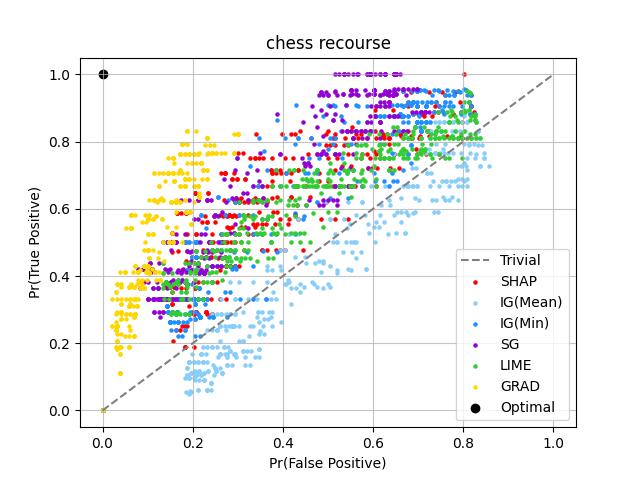}
\end{subfigure}
\begin{subfigure}
    \centering
    \includegraphics[width=.33\linewidth]{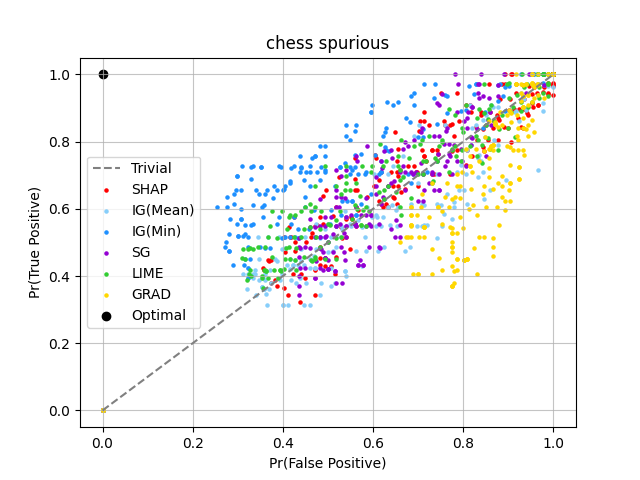}
\end{subfigure}
\begin{subfigure}
    \centering
    \includegraphics[width=.33\linewidth]{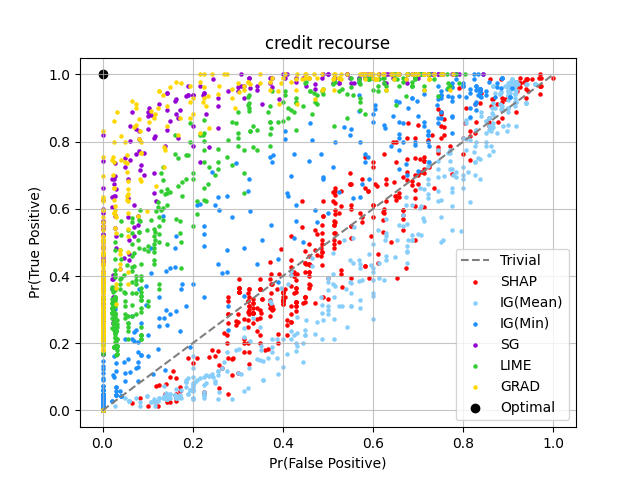}
\end{subfigure}
\begin{subfigure}
    \centering
    \includegraphics[width=.33\linewidth]{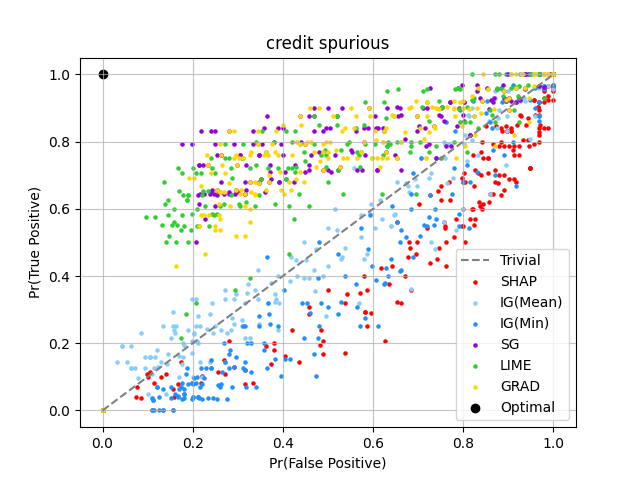}
\end{subfigure}
\begin{subfigure}
    \centering
    \includegraphics[width=.33\linewidth]{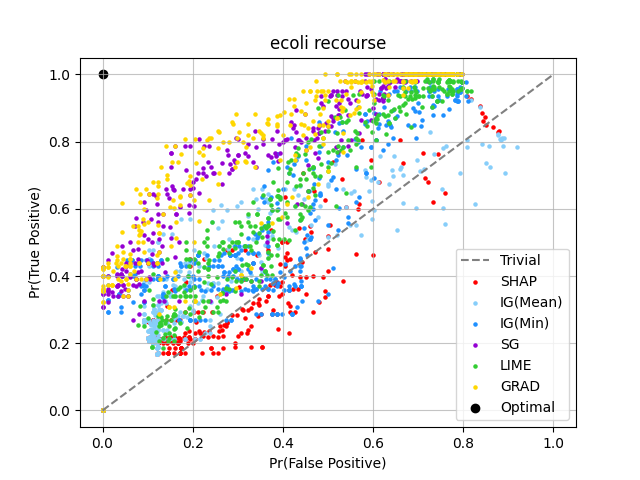}
\end{subfigure}
\begin{subfigure}
    \centering
    \includegraphics[width=.33\linewidth]{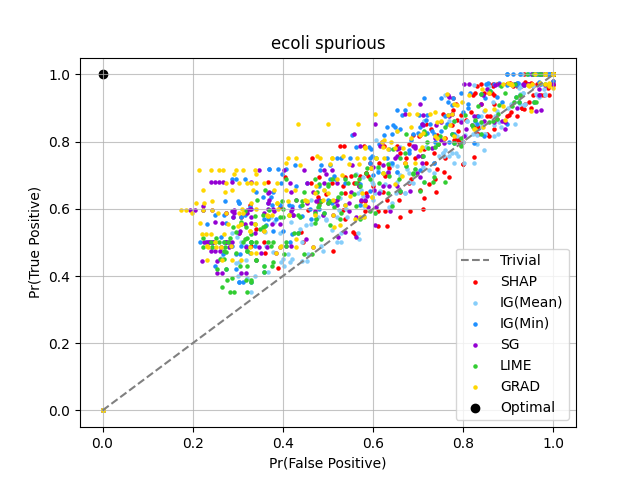}
\end{subfigure}
\begin{subfigure}
    \centering
    \includegraphics[width=.33\linewidth]{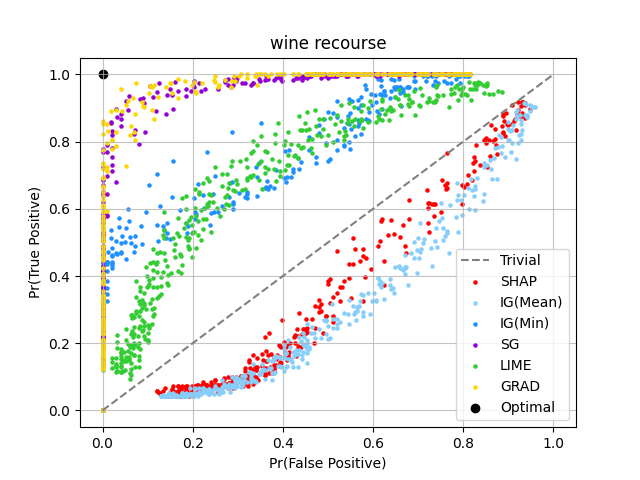}
\end{subfigure}
\begin{subfigure}
    \centering
    \includegraphics[width=.33\linewidth]{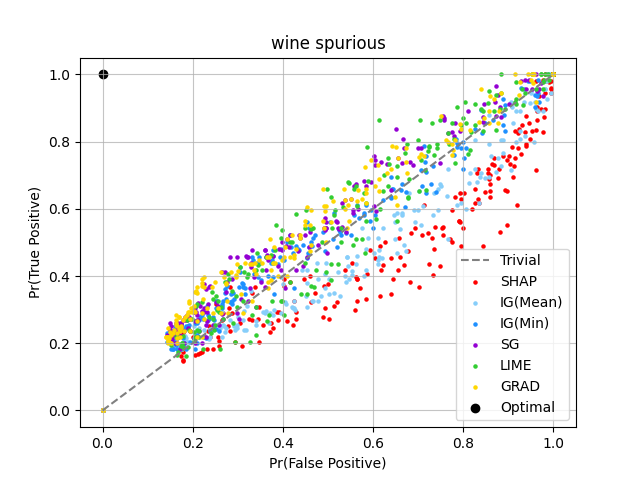}
\end{subfigure}
    \caption{Visualizing ROC curves for tabular datasets. A \fielddefname{} is better for an \downtaskname{} if the ROC curve is closer to the top left corner on average.}
    \label{fig:10percent-plots-1}
\end{figure*}
\newpage

\begin{figure*}[!hpt]
\centering
\begin{subfigure}
    \centering
    \includegraphics[width=.4\linewidth]{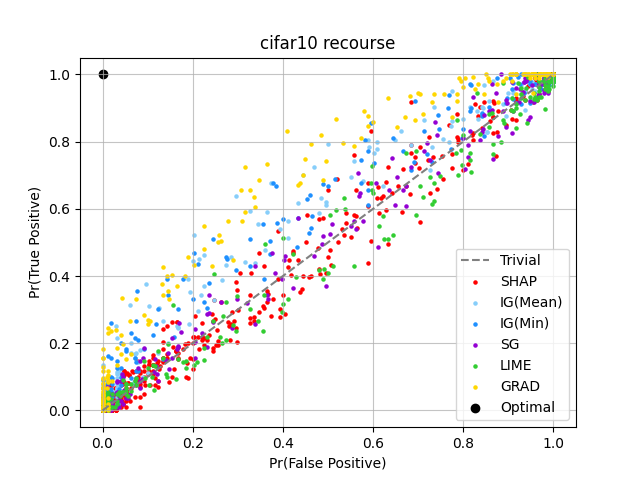}
\end{subfigure}
\begin{subfigure}
    \centering
    \includegraphics[width=.4\linewidth]{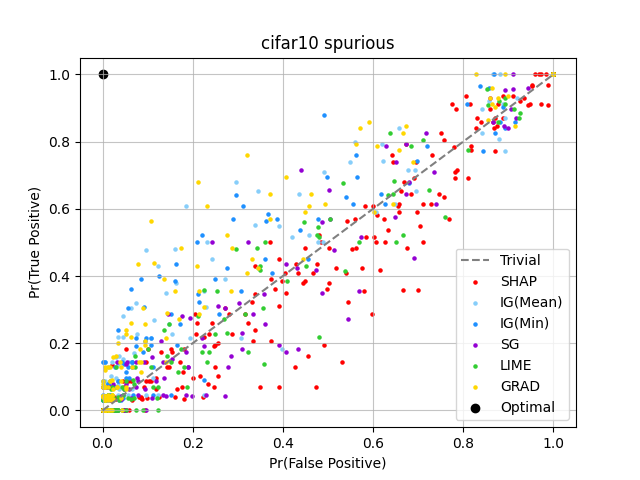}
\end{subfigure}
\begin{subfigure}
    \centering
    \includegraphics[width=.4\linewidth]{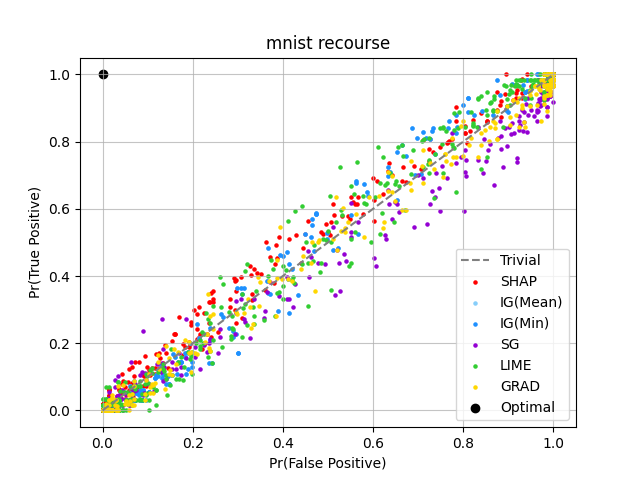}
\end{subfigure}
\begin{subfigure}
    \centering
    \includegraphics[width=.4\linewidth]{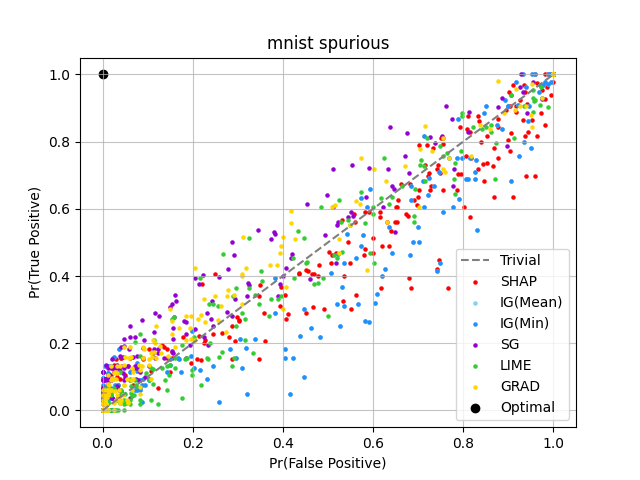}
\end{subfigure}
\begin{subfigure}
    \centering
    \includegraphics[width=.4\linewidth]{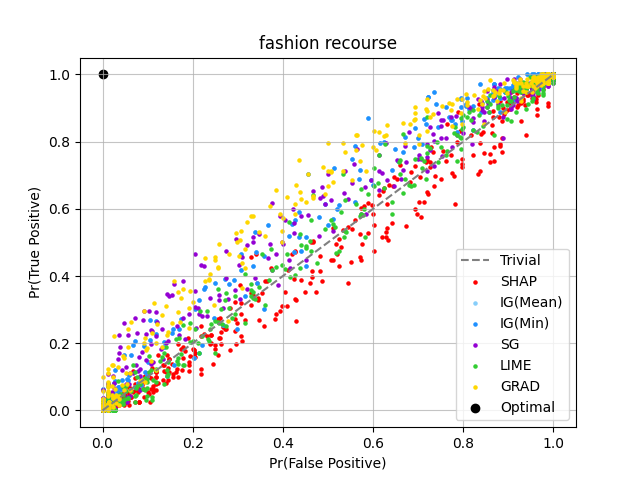}
\end{subfigure}
\begin{subfigure}
    \centering
    \includegraphics[width=.4\linewidth]{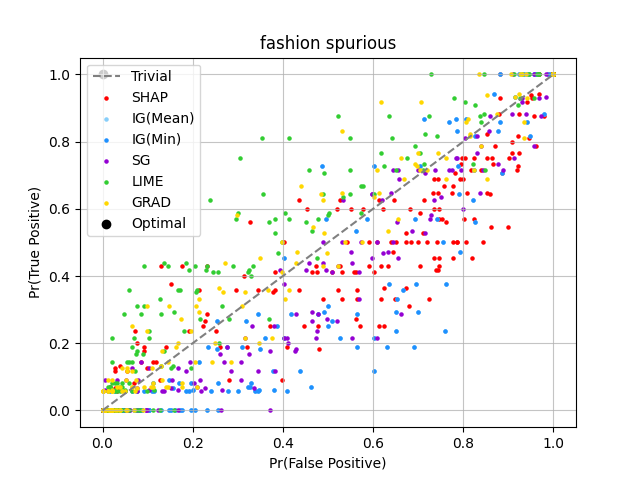}
\end{subfigure}
    \caption{Visualizing ROC curves for image datasets. A \fielddefname{} is better for an \downtaskname{} if the ROC curve is closer to the top left corner on average.}
    \label{fig:10percent-plots-2}
\end{figure*}
\clearpage

\section{Towards Theoretical Guarantees for\\ Perturbation-Based Methods}\label{sec:sample-complexity}

We have shown both theoretically and empirically that \completename{} and \additivename{} \fielddefnames{} like \shap{} and \intgradshort{} can be unreliable for solving \downtasknames{} like algorithmic recourse and spurious \covname{} identification.
At the same time, while simpler \fielddefnames{} like \vanillagrad{} can be accurate for inferring sufficiently \localname{} \modbehaviourname{}, this can also clearly fail even for simple tasks that demand understanding a moderately \localname{} region.
So, faced with solving a given \downtaskname{}, what should a practitioner do?

To answer this question, we turn to the existing literature on \fielddefnames{} that are empirically better than \shap{} and \intgradshort{} for many \tasknames{}.
For example,
\citet{fong2017interpretable} learn a blurring mask that changes the class probability while simultaneously maximizing how ``informative'' the mask is,
\citet{petsiuk18rise} use randomly generated masks to learn which perturbations maximally change the model output,
\citet{kapishnikov19xrai} use image masking to find regions that maximally change the class probability,
\citet{qi2019visualizing} and \citet{khorram2021igos++} learn a feature mask based on both removing and adding components to the image to change the class probability,
and \citet{shitole2021one} combine multiple saliency maps into a single graph to illustrate multiple different minimal perturbations to change the model output.

However, despite the many methods listed above, mathematically establishing when these methods are (and are not) reliable has remained elusive.
As a first step towards such theory, we consider a \methodname{} that is \emph{guaranteed} to work: brute-force solving \downtasknames{} via repeated \modelname{} evaluations.
In particular, since \countmodbehaviourname{} is determined by $\model(\covvaldum)$ for $\covvaldum$ near $\covval$, this can always be inferred by computing $\model(\covvaldum)$ (or, perhaps, $\grad \model(\covvaldum)$) at sufficiently many \obsnames{} $\covvaldum$.
To demonstrate this, we state the following simple result showing how spurious \covname{} identification (i.e., \cref{defn:spurious}) relies on the number of queries.
To avoid additional notational burden, we state this result informally, and defer precise notation, more general statements, and proofs to \cref{sec:sample-complexity-proofs}.

\begin{theorem}\label{fact:sufficient-samples-informal}
Suppose that $\dataspace=\Reals$ and there exists $\lipconst>0$ such that all $\model\in\modelspace$ are $\lipconst$-Lipschitz. 
For fixed $\inscale,\outscale>0$ and $j\in[\covdim]$, consider the \downtaskname{} of \cref{defn:spurious} with $\covval=0$.
For every $n\in\Nats$, there exists a \hyptestname{} $\oraclehyptest\timeind{n}$ that uses only $n$ evaluations of $\model$ yet
\*[
    \specsym(\oraclehyptest\timeind{n}) = 1
    \ \text{ and } \
    \senssym(\oraclehyptest\timeind{n})
    = 1 - \Big(1-\frac{2\outscale}{\lipconst\inscale}\Big)^n.
\]
\end{theorem}

A few remarks on \cref{fact:sufficient-samples-informal}. 
First, this result implies that the \username{} can achieve \specname{} \emph{and} \sensname{} arbitrarily close to one by evaluating $\model$ sufficiently many times (i.e., taking $n$ to infinity).
Second, while we did not state this in \cref{sec:impossibility} for simplicity, from the proof of \cref{fact:oracle-hypothesis} it can be observed that \cref{assn:piecewise} can be relaxed to require only Lipschitz piecewise linear functions, and hence this additional structure is not sufficient to circumvent our impossibility result for \completename{} and \additivename{} \fielddefnames{} (i.e., \cref{fact:oracle-hypothesis-examples}). 
Third, we defer the precise definition of the \hyptestname{} to \cref{sec:sample-complexity-proofs}, but here is a simple intuition: uniformly sample $n$ \obsnames{} $\covval\timeind{1:n}$ and reject the \nullname{} if and only if $\model(0,\covval\feat{j}\timeind{t}) > \outscale$ for at least one $t$.
Finally, in \cref{sec:sample-complexity-proofs} we extend this result to the multivariate case and quantify the corresponding dependence on $\covdim$ (unsurprisingly, this dependence is exponential), and we show that this simple \hyptestname{} achieves nearly optimal worst-case dependence on parameters like $\inscale$, $\outscale$, $\lipconst$, $\covdim$, and $n$. 

As a concrete example, consider the spurious \covname{} identification \downtaskname{} for 10 \covnames{} (say, 10 pixels where a watermark appears) on $1$-Lipschitz \modelnames{}.
To identify if the \modelname{} is sensitive (\modelname{} output changes by more than 1\%) to a 5\% change in the \covname{} value, we show it is possible to get perfect \specname{} and over 90\% \sensname{} with roughly 20,000 model evaluations. 

While the \downtaskname{} of spurious \covname{} identification can be solved by brute-force evaluations (and a similar argument can be made for other \countmodbehaviourname{} \downtasknames{}), there is much work to be done.
In particular, the simple \hyptestname{} in \cref{fact:sufficient-samples-informal} is designed to always succeed, but may be quite inefficient for more structured \modelnames{}.
Unfortunately, for the existing brute-force \hyptestnames{} that are more efficient, it is unclear for which \downtasknames{} they provably work.
Our hypothesis testing framework allows us to rigorously evaluate such methods, and our formalization of \downtasknames{} suggests that techniques and methods from optimization theory may be a useful starting point.

\section{Related Literature}\label{sec:literature}

While theory is sparse for \covname{} \attname{}, some other impossibility results have recently appeared in the literature.

\citet{srinivas19gradient} prove that any \completename{} \fielddefname{} cannot be ``weakly dependent'' on the input. Intuitively, weak dependence is a type of stability, and is precisely defined as the \fielddefname{} output depending only on the weights if the \modelname{} is piecewise linear.
One implication of our results is that any \completename{} and \linearname{} \fielddefname{} cannot even be \emph{close} to weakly dependent, since this would allow a \hyptestname{} to extract useful information about the weights for piecewise linear \modelnames{} (and hence beat random guessing at \tasknames{} like recourse or spurious \covnames{}).

\citet{fokkema22explanations} prove that continuous \fielddefnames{} (including \shap{} and \intgradshort{}) can fail for algorithmic recourse at \obsnames{} near the decision boundary of a \modelname{} (i.e., where the true recourse direction should switch).
In contrast, we prove that for sufficiently rich \modelname{} classes, recourse-like \countmodbehaviourname{} cannot be reliably inferred at \emph{any} \obsname{}.
Moreover, we prove that these \methodnames{} fail to infer general \countmodbehaviourname{}, with recourse impossibility being a corollary of our results.

\citet{han22function} show that for perturbation-based \fielddefnames{} (including \shap{} and \intgradshort{}), there will always exist a neighbourhood where the \covname{} \attname{} does not match the \modelname{} (in the sense of local function approximation).
In contrast, our results imply that this holds for exactly the neighbourhood where the \fielddefname{} is centered; that is, not only do \fielddefnames{} fail \textit{somewhere} to capture \modbehaviourname{} globally, they fail to do so \textit{in the exact local region of interest} as well. 

Beyond such impossibility results, others have proposed various ways to formalize the task of \covname{} \attname{}.
\citet{watson21necessity} use the concepts of necessity and sufficiency to discuss whether \covname{} \attnames{} can be tied to \modbehaviourname{}.
\citet{watson21explanation} propose studying interpretability as a decision theory problem, where two players iterate to find a \scarequo{good} explanation, which they formalize as learning a local approximation of the \modelname{}.
\citet{afchar21feature} formalize ground truth for \fielddefnames{} using a local form of functional feature dependence.
\citet{zhou22solvability} formalize \fielddefnames{} as loss minimizers, which they use to unify certain \methodnames{}.
While all of these formalizations provide a useful lens for studying \fielddefnames{}, none have been used to prove results about their performance.
Our hypothesis testing framework formalizes (some of) these ideas and enables concrete mathematical reasoning about the performance of \fielddefnames{}.

Finally, there exist many individual counterexamples for common \fielddefnames{} in the literature.
\citet{sundararajan20shapley} unify various versions of Shapley value methods and provide an example where \shap{} and \intgradshort{} differ, for which the authors point out that it is not clear which one is correct.
\citet{kumar20shapley} provide simple counterexamples where \shap{} 
can be computed analytically yet result in \attnames{} that disagree with human intuition.
\citet{merrick20explanation} and \citet{janzing20causal} show that the conditional version of \shap{} can give large \attname{} to completely irrelevant \covnames{}. 
Our results are consistent with the above findings, 
providing a way to formally compare \fielddefnames{} (hypothesis testing) along with general, mild conditions under which common \methodnames{} can have poor performance.
Characterizing other \fielddefnames{} in a similar way and refining the assumptions that lead to performance guarantees is critical to better understanding the utility of \covname{} \attname{} as part of a \username{}'s toolkit.

\section{Conclusion}\label{sec:conclusion}

The current deployment of feature attribution methods in high-stakes settings such as medicine and law
demands a rigorous understanding of their performance.
In this work, we characterized conditions under which common \fielddefnames{} provably are unreliable when used to infer \behaviourname{} for learned \modelnames{}. Our work concludes that using feature attribution as is currently prescribed in the literature need not improve on random guessing at inferring model behavior.
The implications of our work are two-fold.
First, it is necessary to build more structure into these methods in order to improve upon the current performance guarantees. We show that a simple brute-force method (easy to define but computationally expensive) is guaranteed to work for \textit{some \downtasknames{}}.
As a result, our second implication is that defining the \downtaskname{} accurately and concretely is crucial. Once given an \downtaskname{}, seeking a \methodname{} that directly optimizes the \taskname{} can provide \usernames{} with straightforward answers. 
In summary, the goals of interpretability are not impossible to achieve, but they may require new \methodnames{} along with a precise understanding of the assumptions under which such \methodnames{} are reliable. 
The development of \methodnames{} that come equipped with performance guarantees and enjoy computational efficiency remains a major open challenge for future work.

\clearpage

\journal{
\acknow{BB acknowledges support from the Vector Institute.
We thank Astrid Bertrand, Robert Geirhos, Adam Pearce, Lisa Schut, Martin Wattenberg, and Qiqi Yan for helpful feedback on preliminary drafts.}

\showacknow{} %

\bibliography{bib-files/interpretability}
}

\preprint{
\section*{Acknowledgements}
BB acknowledges support from the Vector Institute.
We thank Astrid Bertrand, Robert Geirhos, Adam Pearce, Lisa Schut, Martin Wattenberg, and Qiqi Yan for helpful feedback on preliminary drafts.

\bibliography{bib-files/interpretability}

\begin{thebibliography}{50}
\providecommand{\natexlab}[1]{#1}
\providecommand{\url}[1]{\texttt{#1}}
\expandafter\ifx\csname urlstyle\endcsname\relax
  \providecommand{\doi}[1]{doi: #1}\else
  \providecommand{\doi}{doi: \begingroup \urlstyle{rm}\Url}\fi

\bibitem[{Aas} et~al.(2021){Aas}, {Jullum}, and {L{\o}land}]{aas21shapley}
Kjersti {Aas}, Martin {Jullum}, and Anders {L{\o}land}.
\newblock Explaining individual predictions when features are dependent: More
  accurate approximations to {Shapley} values.
\newblock \emph{Artificial Intelligence}, 298:\penalty0 103502--103526, 2021.

\bibitem[Abadi et~al.(2015)Abadi, Agarwal, Barham, Brevdo, Chen, Citro,
  Corrado, Davis, Dean, Devin, Ghemawat, Goodfellow, Harp, Irving, Isard, Jia,
  Jozefowicz, Kaiser, Kudlur, Levenberg, Man\'{e}, Monga, Moore, Murray, Olah,
  Schuster, Shlens, Steiner, Sutskever, Talwar, Tucker, Vanhoucke, Vasudevan,
  Vi\'{e}gas, Vinyals, Warden, Wattenberg, Wicke, Yu, and
  Zheng]{tensorflow2015-whitepaper}
Mart\'{i}n Abadi, Ashish Agarwal, Paul Barham, Eugene Brevdo, Zhifeng Chen,
  Craig Citro, Greg~S. Corrado, Andy Davis, Jeffrey Dean, Matthieu Devin,
  Sanjay Ghemawat, Ian Goodfellow, Andrew Harp, Geoffrey Irving, Michael Isard,
  Yangqing Jia, Rafal Jozefowicz, Lukasz Kaiser, Manjunath Kudlur, Josh
  Levenberg, Dandelion Man\'{e}, Rajat Monga, Sherry Moore, Derek Murray, Chris
  Olah, Mike Schuster, Jonathon Shlens, Benoit Steiner, Ilya Sutskever, Kunal
  Talwar, Paul Tucker, Vincent Vanhoucke, Vijay Vasudevan, Fernanda Vi\'{e}gas,
  Oriol Vinyals, Pete Warden, Martin Wattenberg, Martin Wicke, Yuan Yu, and
  Xiaoqiang Zheng.
\newblock {TensorFlow}: Large-scale machine learning on heterogeneous systems,
  2015.
\newblock URL \url{https://www.tensorflow.org/}.
\newblock Software available from tensorflow.org.

\bibitem[{Adebayo} et~al.(2018){Adebayo}, {Gilmer}, {Muelly}, {Goodfellow},
  {Hardt}, and {Kim}]{adebayo18sanity}
Julius {Adebayo}, Justin {Gilmer}, Michael {Muelly}, Ian {Goodfellow}, Moritz
  {Hardt}, and Been {Kim}.
\newblock Sanity checks for saliency maps.
\newblock In \emph{Advances in Neural Information Processing Systems 32}, 2018.

\bibitem[{Afchar} et~al.(2021){Afchar}, {Hennequin}, and
  {Guigue}]{afchar21feature}
Darius {Afchar}, Romain {Hennequin}, and Vincent {Guigue}.
\newblock Towards rigorous interpretations: A formalisation of feature
  attribution.
\newblock In \emph{Proceedings of the 38th International Conference on Machine
  Learning}, 2021.

\bibitem[{Arora} et~al.(2018){Arora}, {Basu}, {Mianjy}, and
  {Mukherjee}]{arora18relu}
Raman {Arora}, Amitabh {Basu}, Poorya {Mianjy}, and Anirbit {Mukherjee}.
\newblock Understanding deep neural networks with rectified linear units.
\newblock In \emph{Proceedings of the 6th International Conference on Learning
  Representations}, 2018.

\bibitem[Bain and Muggleton(1994)]{bain1994chess}
Michael Bain and Stephen Muggleton.
\newblock Learning optimal chess strategies.
\newblock In \emph{Machine Intelligence 13: Machine Intelligence and Inductive
  Learning}, pages 291--309. 1994.

\bibitem[{Chen} et~al.(2022){Chen}, {Garudadri}, and {Rao}]{chen22piecewise}
Kuan-Lin {Chen}, Harinath {Garudadri}, and Bhaskar~D. {Rao}.
\newblock Improved bounds on neural complexity for representing piecewise
  linear functions.
\newblock In \emph{Advances in Neural Information Processing Systems 36}, 2022.

\bibitem[Dua and Graff(2017)]{Dua19UCI}
Dheeru Dua and Casey Graff.
\newblock {UCI} machine learning repository, 2017.
\newblock URL \url{http://archive.ics.uci.edu/ml}.

\bibitem[{Fokkema} et~al.(2022){Fokkema}, {de Heide}, and {van
  Erven}]{fokkema22explanations}
Hidde {Fokkema}, Rianne {de Heide}, and Tim {van Erven}.
\newblock Attribution-based explanations that provide recourse cannot be
  robust, 2022.
\newblock arXiv:2205.15834.

\bibitem[Fong and Vedaldi(2017)]{fong2017interpretable}
Ruth~C Fong and Andrea Vedaldi.
\newblock Interpretable explanations of black boxes by meaningful perturbation.
\newblock In \emph{Proceedings of the 2017 IEEE International Conference on
  Computer Vision}, 2017.

\bibitem[{Forina} et~al.(1986){Forina}, {Armanino}, {Castino}, and
  {Ubigli}]{forina86wine}
M.~{Forina}, C.~{Armanino}, M.~{Castino}, and M.~{Ubigli}.
\newblock Multivariate data analysis as a discriminating method of the origin
  of wines.
\newblock \emph{Journal of Grapevine Research}, 25\penalty0 (3), 1986.

\bibitem[{Ghosh} et~al.(2022){Ghosh}, {Shanbhag}, and {Wilson}]{ghosh22fair}
Avijit {Ghosh}, Aalok {Shanbhag}, and Christo {Wilson}.
\newblock Faircanary: Rapid continuous explainable fairness.
\newblock In \emph{Proceedings of the 2022 AAAI/ACM Conference on AI, Ethics,
  and Society}, 2022.

\bibitem[{Han} et~al.(2022){Han}, {Srinivas}, and {Lakkaraju}]{han22function}
Tessa {Han}, Suraj {Srinivas}, and Himabindu {Lakkaraju}.
\newblock Which explanation should {I} choose? {A} function approximation
  perspective to characterizing post hoc explanations.
\newblock In \emph{Advances in Neural Information Processing Systems 36}, 2022.

\bibitem[{Jain} et~al.(2020){Jain}, {Ravula}, and {Ghosh}]{aditya20biased}
Aditya {Jain}, Manish {Ravula}, and Joydeep {Ghosh}.
\newblock Baised models have biased explanations, 2020.
\newblock arXiv:2012.10986.

\bibitem[{Janzing} et~al.(2020){Janzing}, {Minorics}, and
  {Bl\"{o}baum}]{janzing20causal}
Dominik {Janzing}, Lenon {Minorics}, and Patrick {Bl\"{o}baum}.
\newblock Feature relevance quantification in explainable {AI}: A causal
  problem.
\newblock In \emph{Proceedings of the 23rd International Conference on
  Artificial Intelligence and Statistics}, 2020.

\bibitem[{Kapishnikov} et~al.(2019){Kapishnikov}, {Bolukbasi}, {Viegas}, and
  {Terry}]{kapishnikov19xrai}
Andrei {Kapishnikov}, Tolga {Bolukbasi}, Fernanda {Viegas}, and Michael
  {Terry}.
\newblock {XRAI}: Better attributions through regions.
\newblock In \emph{2019 IEEE/CVF International Conference on Computer Vision},
  2019.

\bibitem[Khorram et~al.(2021)Khorram, Lawson, and Fuxin]{khorram2021igos++}
Saeed Khorram, Tyler Lawson, and Li~Fuxin.
\newblock {iGOS++} integrated gradient optimized saliency by bilateral
  perturbations.
\newblock In \emph{Proceedings of the 2021 Conference on Health, Inference, and
  Learning}, 2021.

\bibitem[Krizhevsky(2009)]{Krizhevsky09learningmultiple}
Alex Krizhevsky.
\newblock Learning multiple layers of features from tiny images.
\newblock Technical report, 2009.

\bibitem[{Kumar} et~al.(2020){Kumar}, {Venkatasubramanian}, {Scheidegger}, and
  {Friedler}]{kumar20shapley}
I.~Elizabeth {Kumar}, Suresh {Venkatasubramanian}, Carlos {Scheidegger}, and
  Sorelle~A. {Friedler}.
\newblock Problems with {Shapley}-value-based explanations as feature
  importance measures.
\newblock In \emph{Proceedings of the 37th International Conference on Machine
  Learning}, 2020.

\bibitem[LeCun et~al.(2010)LeCun, Cortes, and Burges]{lecun2010mnist}
Yann LeCun, Corinna Cortes, and CJ~Burges.
\newblock {MNIST} handwritten digit database.
\newblock \emph{ATT Labs [Online]. Available:
  http://yann.lecun.com/exdb/mnist}, 2, 2010.

\bibitem[{Lerma} and {Lucas}(2021)]{lerma21integrated}
Miguel {Lerma} and Mirtha {Lucas}.
\newblock Symmetry-preserving paths in integrated gradients, 2021.
\newblock arXiv:2103.13533.

\bibitem[{Liu} et~al.(2021){Liu}, {Rizzo}, {Whipple}, {Pal}, {Pineda}, {Lu},
  {Arnieri}, {Lu}, {Capra}, {Copping}, and {Zou}]{liu21oncology}
Ruishan {Liu}, Shemra {Rizzo}, Samuel {Whipple}, Navdeep {Pal}, Arturo~Lopez
  {Pineda}, Michael {Lu}, Brandon {Arnieri}, Ying {Lu}, William {Capra}, Ryan
  {Copping}, and James {Zou}.
\newblock Evaluating eligibility criteria of oncology trials using real-world
  data and {AI}.
\newblock \emph{Nature}, 592:\penalty0 629--633, 2021.

\bibitem[{Lundberg} and {Lee}(2017)]{lundberg17shapley}
Scott~M. {Lundberg} and Su-In {Lee}.
\newblock A unified approach to interpreting model predictions.
\newblock In \emph{Advances in Neural Information Processing Systems 31}, 2017.

\bibitem[{McCloskey} et~al.(2019){McCloskey}, {Taly}, {Monti}, {Brenner}, and
  {Colwell}]{mccloskey19attribution}
Kevin {McCloskey}, Ankur {Taly}, Frederico {Monti}, Michael~P. {Brenner}, and
  Lucy~J. {Colwell}.
\newblock Using attribution to decode binding mechanism in neural network
  models for chemistry.
\newblock \emph{Proceedings of the National Academy of Sciences of the United
  States of America}, 116\penalty0 (24):\penalty0 11624--11629, 2019.

\bibitem[{Merrick} and {Taly}(2020)]{merrick20explanation}
Luke {Merrick} and Ankur {Taly}.
\newblock The explanation game: Explaining machine learning models using
  {Shapley} values.
\newblock In \emph{International Cross-Domain Conference for Machine Learning
  and Knowledge Extraction}, 2020.

\bibitem[Molnar(2022)]{molnar2022}
Christoph Molnar.
\newblock \emph{Interpretable Machine Learning}.
\newblock 2 edition, 2022.
\newblock URL \url{https://christophm.github.io/interpretable-ml-book}.

\bibitem[Nakai and Kanehisa(1991)]{nakai1991ecoli}
Kenta Nakai and Minoru Kanehisa.
\newblock Expert system for predicting protein localization sites in
  gram-negative bacteria.
\newblock \emph{Proteins: Structure, Function, and Bioinformatics}, 11\penalty0
  (2):\penalty0 95--110, 1991.

\bibitem[Nash et~al.(1994)Nash, Sellers, Talbot, Cawthorn, and
  Ford]{nash1994abalone}
Warwick~J Nash, Tracy~L Sellers, Simon~R Talbot, Andrew~J Cawthorn, and Wes~B
  Ford.
\newblock The population biology of {Abalone} ({\emph{haliotis}} species) in
  {Tasmania}. i. {Blacklip Abalone} ({\emph{h. rubra}}) from the north coast
  and islands of {Bass Strait}.
\newblock \emph{Sea Fisheries Division, Technical Report}, 48:\penalty0 p411,
  1994.

\bibitem[{Nie} et~al.(2018){Nie}, {Zhang}, and {Patel}]{nie18backprop}
Weili {Nie}, Yang {Zhang}, and Ankit~B. {Patel}.
\newblock A theoretical explanation for perplexing behaviors of
  backpropagation-based visualizations.
\newblock In \emph{Proceedings of the 35th International Conference on Machine
  Learning}, 2018.

\bibitem[{Petsiuk} et~al.(2018){Petsiuk}, {Das}, and {Saenko}]{petsiuk18rise}
Vitali {Petsiuk}, Abir {Das}, and Kate {Saenko}.
\newblock {RISE}: Randomized input sampling for explanation of black-box
  models.
\newblock In \emph{Proceedings of the 29th British Machine Vision Conference},
  2018.

\bibitem[Qi et~al.(2019)Qi, Khorram, and Li]{qi2019visualizing}
Zhongang Qi, Saeed Khorram, and Fuxin Li.
\newblock Visualizing deep networks by optimizing with integrated gradients.
\newblock In \emph{Workshop at the 34th IEEE Conference on Computer Vision and
  Pattern Recognition}, 2019.

\bibitem[Quinlan(1987)]{quinlan1987credit}
J.~Ross Quinlan.
\newblock Simplifying decision trees.
\newblock \emph{International journal of man-machine studies}, 27\penalty0
  (3):\penalty0 221--234, 1987.

\bibitem[{Ribeiro} et~al.(2016){Ribeiro}, {Singh}, and
  {Guestrin}]{ribeiro16lime}
Marco~Tulio {Ribeiro}, Sameer {Singh}, and Carlos {Guestrin}.
\newblock {``Why} should {I} trust you?" {Explaining} the predictions of any
  classifier.
\newblock In \emph{Proceedings of the 22nd ACM SIGKDD Conference on Knowledge
  Discovery and Data Mining}, 2016.

\bibitem[{Roder} et~al.(2021){Roder}, {Maguire}, {Georgantas III}, and
  {Roder}]{roder21molecular}
Joanna {Roder}, Laura {Maguire}, Robert {Georgantas III}, and Heinrich {Roder}.
\newblock Explaining multivariate molecular diagnostic tests via {Shapley}
  values.
\newblock \emph{BMC Medical Informatics and Decision Making}, 21\penalty0
  (211), 2021.

\bibitem[Shitole et~al.(2021)Shitole, Li, Kahng, Tadepalli, and
  Fern]{shitole2021one}
Vivswan Shitole, Fuxin Li, Minsuk Kahng, Prasad Tadepalli, and Alan Fern.
\newblock One explanation is not enough: Structured attention graphs for image
  classification.
\newblock \emph{Advances in Neural Information Processing Systems 35}, 2021.

\bibitem[{Simonyan} et~al.(2013){Simonyan}, {Vedaldi}, and
  {Zisserman}]{simonyan13gradient}
Karen {Simonyan}, Andrea {Vedaldi}, and Andrew {Zisserman}.
\newblock Deep inside convolutional networks: Visualising image classification
  models and saliency maps, 2013.
\newblock arXiv:1312.6034.

\bibitem[{Smilkov} et~al.(2017){Smilkov}, {Thorat}, {Kim}, {Viegas}, and
  {Wattenberg}]{smilkov17smoothgrad}
Daniel {Smilkov}, Nikhil {Thorat}, Been {Kim}, Fernanda {Viegas}, and Martin
  {Wattenberg}.
\newblock {SmoothGrad}: Removing noise by adding noise.
\newblock In \emph{Proceedings of the ICML 2017 Workshop on Visualization for
  Deep Learning}, 2017.

\bibitem[{Srinivas} and {Fleuret}(2019)]{srinivas19gradient}
Suraj {Srinivas} and Fran\c{c}ois {Fleuret}.
\newblock Full-gradient representation for neural network visualization.
\newblock In \emph{Advances in Neural Information Processing Systems 33}, 2019.

\bibitem[Sturmfels et~al.(2020)Sturmfels, Lundberg, and
  Lee]{sturmfels2020visualizing}
Pascal Sturmfels, Scott Lundberg, and Su-In Lee.
\newblock Visualizing the impact of feature attribution baselines.
\newblock \emph{Distill}, 2020.
\newblock \doi{10.23915/distill.00022}.
\newblock https://distill.pub/2020/attribution-baselines.

\bibitem[{Sundararajan} and {Najmi}(2020)]{sundararajan20shapley}
Mukund {Sundararajan} and Amir {Najmi}.
\newblock The many {Shapley} values for model explanation.
\newblock In \emph{Proceedings of the 37th International Conference on Machine
  Learning}, 2020.

\bibitem[{Sundararajan} et~al.(2017){Sundararajan}, {Taly}, and
  {Yan}]{sundararajan17integrated}
Mukund {Sundararajan}, Ankur {Taly}, and Qiyi {Yan}.
\newblock Axiomatic attribution for deep networks.
\newblock In \emph{Proceedings of the 34th International Conference on Machine
  Learning}, 2017.

\bibitem[{Tanielian} et~al.(2021){Tanielian}, {Sangnier}, and
  {Biau}]{tanielian21groupsort}
Ugo {Tanielian}, Maxime {Sangnier}, and Gerard {Biau}.
\newblock Approximating {Lipschitz} continuous functions with {GroupSort}
  neural networks.
\newblock In \emph{Proceedings of the 24th International Conference on
  Artificial Intelligence and Statistics}, 2021.

\bibitem[{Watson} and {Floridi}(2021)]{watson21explanation}
David~S. {Watson} and Luciano {Floridi}.
\newblock The explanation game: A formal framework for interpretable machine
  learning.
\newblock \emph{Synthese}, 198:\penalty0 9211--9242, 2021.

\bibitem[{Watson} et~al.(2021){Watson}, {Gultchin}, {Taly}, and
  {Floridi}]{watson21necessity}
David~S. {Watson}, Limor {Gultchin}, Ankur {Taly}, and Luciano {Floridi}.
\newblock Local explanations via necessity and sufficiency: Unifying theory and
  practice.
\newblock In \emph{Proceedings of the Thirty-Seventh Conference on Uncertainty
  in Artificial Intelligence}, 2021.

\bibitem[Xiao et~al.(2017)Xiao, Rasul, and Vollgraf]{xiao17fashion}
Han Xiao, Kashif Rasul, and Roland Vollgraf.
\newblock {Fashion-MNIST}: a novel image dataset for benchmarking machine
  learning algorithms, 2017.
\newblock arXiv:1708.07747.

\bibitem[{Yeh} et~al.(2019){Yeh}, {Hsieh}, {Suggala}, {Inouye}, and
  {Ravikumar}]{yeh19infidelity}
Chih-Kuan {Yeh}, Cheng-Yu {Hsieh}, Arun~Sai {Suggala}, David~I. {Inouye}, and
  Pradeep {Ravikumar}.
\newblock On the (in)fidelity and sensitivity of explanations.
\newblock In \emph{Advances in Neural Information Processing Systems 33}, 2019.

\bibitem[{Yerushalmy}(1947)]{yerushalmy47hyptest}
Jacob {Yerushalmy}.
\newblock Statistical problems in assessing methods of medical diagnosis, with
  special reference to {X}-ray techniques.
\newblock \emph{Public Health Reports (1896-1970)}, 62\penalty0 (40):\penalty0
  1432--1449, 1947.

\bibitem[{Zaeri-Amirani} et~al.(2018){Zaeri-Amirani}, {Afghah}, and
  {Mousavi}]{zaeri18icus}
Mohammad {Zaeri-Amirani}, Fatemeh {Afghah}, and Sajad {Mousavi}.
\newblock A feature selection method based on {Shapley} value to false alarm
  reduction in {ICUs}, a genetic-algorithm approach.
\newblock In \emph{Proceedings of the 40th Annual International Conference of
  the IEEE Engineering in Medicine and Biology Society}, 2018.

\bibitem[{Zhou} et~al.(2021){Zhou}, {Arslanturk}, {Craig}, {Heath}, and
  {Draghici}]{zhou21prostate}
Kaiyue {Zhou}, Suzan {Arslanturk}, Douglas~B. {Craig}, Elisabeth {Heath}, and
  Sorin {Draghici}.
\newblock Discovery of primary prostate cancer biomarkers using cross cancer
  learning.
\newblock \emph{Scientific Reports}, 11\penalty0 (10433), 2021.

\bibitem[{Zhou} and {Shah}(2022)]{zhou22solvability}
Yilun {Zhou} and Julie {Shah}.
\newblock The solvability of interpretability evaluation metrics, 2022.
\newblock arXiv:2205.08696.

\end{thebibliography}
}

\newpage
\journal{\onecolumn}

\appendix

\section{Proofs for \completeNames{} and \additiveNames{}}\label{sec:complete-additive-proofs}

We begin with precise definitions of \shap{} and \intgrad{}.

For \shap{}, there are two possible definitions in the literature, which respectively use marginal and conditional expectations.
In practice, the most common implementation of \shap{} uses the marginal \defname{} for convenience \citep{kumar20shapley}, and this is what we focus on here. The original definition in \citet{lundberg17shapley} was actually the conditional one (although they apply the marginal definition in experiments, again for tractability), and recent work has studied approximations to compute it \citep{aas21shapley}.
We expect that similar results to \cref{fact:oracle-hypothesis} could be proved in this setting, but this would require a different proof technique and hence we defer it to future work.

\begin{definition}[\shap{} \citep{lundberg17shapley}]\label{def:genshap}
Let
\*[
    \shapwt(i)
    = \frac{i!(\covdim-i-1)!}{\covdim!}.
\]
For all $\model\in\modelspace$, $\covdist\in\probspace(\covspace)$, $\covval\in\covspace$, and $j\in[\covdim]$, define
\*[
    \shapdef(\model,\covdist,\covval)\feat{j}
    =  \EE_{\covobs\sim\covdist} \Bigg[\sum_{\featuresub \subseteq [\covdim]}
    \shapwt(\abs{\featuresub})
    \Big(
    \model(\covval\feat{\featuresub\cup\{j\}},\covobs\feat{\comp{\featuresub}\setminus\{j\}})
    -
    \model(\covval\feat{\featuresub},\covobs\feat{\comp{\featuresub}})
    \Big)\Bigg].
\]
\end{definition}

For \intgrad{}, it is usually only defined with a single \basename{} \obsname{}. In order to state our results most generally, we consider the natural extension of averaging over a \basename{} distribution, which recovers the original definition when the \basename{} is a pointmass.

\begin{definition}[\intgrad{} \citep{sundararajan17integrated}]\label{def:intgrad}
For all $\model\in\modelspace$, $\covdist\in\probspace(\covspace)$, $\covval\in\covspace$, and $j\in[\covdim]$, whenever $\model$ is appropriately differentiable define
\*[
    \intgraddef(\model,\covdist,\covval)\feat{j}
    =
    \EE_{\covobs\sim\covdist}\Bigg[(\covval\feat{j} - \covobs\feat{j}) \int_0^1 \grad\feat{j}\model(\covobs + \alpha(\covval-\covobs)) \, \dee \alpha \Bigg].
\]
\end{definition}

Next, we have the precise statements and proofs of the informal claims at the start of \cref{sec:impossibility}: \shap{} and \intgrad{} are both \completename{} and \additivename{}.

\begin{proposition}\label{fact:complete-defs}
\shap{} is \completename{}.
\end{proposition}
\begin{proof}[Proof of \cref{fact:complete-defs}]

For all $\covval, \covvaldum \in \covspace$,
\*[
    &\hspace{-1em}\sum_{j\in[\covdim]} \sum_{\featuresub \subseteq [\covdim]}
    \shapwt(\abs{\featuresub})
    \Big(
    \model(\covval\feat{\featuresub\cup\{j\}},\covvaldum\feat{\comp{\featuresub}\setminus\{j\}})
    -
    \model(\covval\feat{\featuresub},\covvaldum\feat{\comp{\featuresub}})
    \Big)\\
    &= 
    \sum_{\featuresub \subseteq [\covdim]} \sum_{j\in\featuresub}
    \shapwt(\abs{\featuresub}-1) \model(\covval\feat{\featuresub},\covvaldum\feat{\comp{\featuresub}})
    -
    \sum_{\featuresub \subseteq [\covdim]} \sum_{j\not\in\featuresub}
    \shapwt(\abs{\featuresub}) \model(\covval\feat{\featuresub},\covvaldum\feat{\comp{\featuresub}}) \\
    &= 
    \sum_{\featuresub \subseteq [\covdim]} \model(\covval\feat{\featuresub},\covvaldum\feat{\comp{\featuresub}}) 
    \Big(\abs{\featuresub} \shapwt(\abs{\featuresub}-1) - (\covdim-\abs{\featuresub}) \shapwt(\abs{\featuresub}) \Big) \\
    &=
    \covdim \cdot \shapwt(\covdim-1) \model(\covval) - \covdim \cdot \shapwt(0) \model(\covvaldum)
    +
    \sum_{\featuresub \subseteq [\covdim], \featuresub\not\in\{\emptyset, [\covdim]\}} \model(\covval\feat{\featuresub},\covvaldum\feat{\comp{\featuresub}}) 
    \Big(\abs{\featuresub} \shapwt(\abs{\featuresub}-1) - (\covdim-\abs{\featuresub}) \shapwt(\abs{\featuresub}) \Big).
\]

Then, observe that
\*[
    \shapwt(\covdim-1) = \shapwt(0) = \frac{1}{\covdim},
\]
and for $0 < i < \covdim$,
\*[
    i \cdot \shapwt(i-1)
    &= i \frac{(i-1)!(\covdim-i)!}{\covdim!}
    = \frac{i!(\covdim-i)!}{\covdim!} \quad \text{ and } \\
    (\covdim-i)\cdot\shapwt(i)
    &= (\covdim-i) \frac{i!(\covdim-i-1)!}{\covdim!}
    = \frac{i!(\covdim-i)!}{\covdim!}.
\]

The result follows by replacing $\covvaldum$ with $\covobs\sim\covdist$ and taking expectation.
\end{proof}

\begin{proposition}[\citep{lerma21integrated}]
If for all $\model\in\modelspace$, either
(a) the $\model$ is continuously differentiable everywhere; or (b)
$\covspace$ is open, there exists $\lipconst$ such that $\model$ is $\lipconst$-Lipschitz, and for all $\covval,\covrefval\in\covspace$ and almost every $\alpha\in[0,1]$, $\model$ is differentiable at $\covrefval+\alpha(\covval-\covrefval)$,
then \intgrad{} is \completename{}.
\end{proposition}

\begin{remark}
This is a weaker statement than Proposition~1 of \citet{sundararajan17integrated}, which \citet{lerma21integrated} show is false in general.
\end{remark}

\begin{proposition}\label{fact:additive-defs}
\intgrad{} and \shap{} are \additivename{}.
\end{proposition}
\begin{proof}[Proof of \cref{fact:additive-defs}]
Fix $\model\feat{1},\dots,\model\feat{\covdim}:\Reals\to\dataspace$ and $\model(\covval) = \sum_{j\in[\covdim]} \model\feat{j}(\covval\feat{j})$, $\covdist\in\probspace(\covspace)$, and $\covval\in\covspace$.

For \intgrad{}, for all $\covrefval\in\covspace$ and $j\in[\covdim]$,
\*[
    \intgraddef(\model,\covrefval,\covval)\feat{j}
    &=
    (\covval\feat{j} - \covrefval\feat{j}) \int_0^1 \grad\feat{j}\model(\covrefval + \alpha(\covval-\covrefval)) \, \dee \alpha \\
    &= 
    (\covval\feat{j} - \covrefval\feat{j}) \int_0^1 \grad\feat{j}\sum_{\ell\in[\covdim]}\model\feat{\ell}(\covrefval\feat{\ell} + \alpha(\covval\feat{\ell}-\covrefval\feat{\ell})) \, \dee \alpha \\
    &= 
    (\covval\feat{j} - \covrefval\feat{j}) \int_0^1 \grad\feat{j}\model\feat{j}(\covrefval\feat{j} + \alpha(\covval\feat{j}-\covrefval\feat{j})) \, \dee \alpha \\
    &= \intgraddef(\model\feat{j},\covrefval\feat{j},\covval\feat{j}).
\]
The result follows from taking expectation under $\covdist$.

For \shap{},
\*[
    \shapdef(\model,\covdist,\covval)\feat{j}
    &=  \EE_{\covobs\sim\covdist} \Bigg[\sum_{\featuresub \subseteq [\covdim]}
    \shapwt(\abs{\featuresub})
    \Big(
    \model(\covval\feat{\featuresub\cup\{j\}},\covobs\feat{\comp{\featuresub}\setminus\{j\}})
    -
    \model(\covval\feat{\featuresub},\covobs\feat{\comp{\featuresub}})
    \Big)\Bigg] \\
    &= \EE_{\covobs\sim\covdist} \Bigg[\sum_{\featuresub \subseteq [\covdim]\setdelim j\not\in\featuresub}
    \shapwt(\abs{\featuresub})
    \Big(
    \sum_{i\in\featuresub}\model\feat{i}(\covval\feat{i})
    + \model\feat{j}(\covval\feat{j})
    + \sum_{\ell\in\comp{\featuresub}\setminus\{j\}}\model\feat{\ell}(\covobs\feat{\ell}) \\
    &\qquad\qquad\qquad\qquad\qquad\qquad- \sum_{i\in\featuresub}\model\feat{i}(\covval\feat{i})
    - \model\feat{j}(\covobs\feat{j})
    - \sum_{\ell\in\comp{\featuresub}\setminus\{j\}}\model\feat{\ell}(\covobs\feat{\ell}) 
    \Big)\Bigg]  \\
    &= \model\feat{j}(\covval\feat{j}) - \EE_{\covobs\sim\covdist} \model\feat{j}(\covobs\feat{j}) \\
    &= \shapdef(\model\feat{j}, \covdist\feat{j}, \covval\feat{j}),
\]
where the last step holds since \shap{} is \completename{}. 
\end{proof}

\subsection{Additional \fielddefNAMEs{}}

In \cref{sec:experiments}, we compare \shap{} and \intgrad{} to other \fielddefnames{} that \emph{do not} satisfy \completenames{} and \additivenames{}. To keep the paper self-contained, we redefine these here.

\begin{definition}[\smoothgrad{} \citep{smilkov17smoothgrad}]\label{def:smoothgrad}
For all $\model\in\modelspace$, $\covdist\in\probspace(\covspace)$, $\covval\in\covspace$, and $j\in[\covdim]$, whenever $\model$ is appropriately differentiable define
\*[
    \smoothgraddef(\model,\covdist,\covval)\feat{j}
    =
    \EE_{\covvaldum\sim\covdist} \grad_{j} \model(\covvaldum).
\]
\end{definition}

\begin{definition}[\lime{} \citep{ribeiro16lime}]\label{def:lime}
Fix $\lambda > 0$.
For all $\model\in\modelspace$, $\covdist\in\probspace(\covspace)$, and $\covval\in\covspace$, define
\*[
    \limedef(\model,\covdist,\covval)
    =
    \argmin_{\beta\in\Reals^\covdim}
    \Big[\EE_{\covvaldum\sim\covdist}\Big(\inner[0]{\beta}{\covvaldum} - \model(\covvaldum)\Big)^2 + \lambda \norm{\beta}_2^2 \Big].
\]
\end{definition}

\section{Proofs for Impossibility Results}\label{sec:impossibility-proofs}

\subsection{Generalized Assumptions}

First, we restate \cref{assn:inscale-covdist,assn:piecewise} in the multivariate case to allow for the most general version of our results.

\begin{appassumption}\label{assn:inscale-covdist-multivariate}
For any $\covval\in\covspace$, $\featuresub\subseteq[\covdim]$, $\inscale\in\PosReals^{\abs{\featuresub}}$, and $\covdist\in\probspace(\covspace)$, we say the present assumption holds if \cref{assn:inscale-covdist} holds for each $j\in\featuresub$ with $\inscale\feat{j}$.
\end{appassumption}

\begin{appassumption}\label{assn:piecewise-multivariate}
For any $\covval\in\covspace$, $\featuresub\subseteq[\covdim]$, and $\inscale\in\PosReals^{\abs{\featuresub}}$,
let $\covnbhd = \prod_{j\in\featuresub} [\covval\feat{j}-\inscale\feat{j}, \covval\feat{j}+\inscale\feat{j}] \times \covval\feat{[\covdim]\setminus\featuresub}$.
For any $(\modeldum\feat{j}:[\covval\feat{j}-\inscale\feat{j}, \covval\feat{j}+\inscale\feat{j}]\to\dataspace)_{j\in\featuresub}$ and $\numpiece\in\Nats$, we say that the present assumption holds with size $\numpiece$ if
\*[
    \Big\{\model\in(\covspace\to\dataspace) \setdelim
    \forall\covvaldum\in\covnbhd \quad \model(\covvaldum) = \sum_{j\in\featuresub}\modeldum\feat{j}(\covvaldum\feat{j}), \ \model\restrict{\comp{\covnbhd}} \in \pwiselin{\numpiece}\restrict{\comp{\covnbhd}}, 
    \ \model \text{ is continuous}\Big\} \subseteq \modelspace.
\]
\end{appassumption}

\subsection{Proof of Main Result}

We state and prove the following multivariate generalization.

\begin{theorem}\label{fact:oracle-hypothesis-multivariate}
Fix any $\covval\in\covspace$, $\featuresub\subseteq[\covdim]$, 
$\inscale \in\PosReals^{\abs{\featuresub}}$,
$\covdist \in \probspace(\covspace)$, 
and
$(\modeldum\nullind\feat{j},\modeldum\altind\feat{j}:[\covval\feat{j}-\inscale\feat{j}, \covval\feat{j}+\inscale\feat{j}]\to\dataspace)_{j\in\featuresub}$.
Suppose that \cref{assn:inscale-covdist-multivariate} is satisfied and \cref{assn:piecewise-multivariate} is satisfied with $m=2^{\abs{\featuresub}}$.
Define $\covnbhd$ as in \cref{assn:piecewise-multivariate}, and let
\*[
    \modelspace\nullind
    &= \Big\{\model\in\modelspace\setdelim \forall\covvaldum\in\covnbhd, \ \model(\covvaldum) = \sum_{j\in\featuresub}\modeldum\nullind\feat{j}(\covvaldum\feat{j}) \Big\} \\
    \modelspace\altind
    &= \Big\{\model\in\modelspace\setdelim \forall\covvaldum\in\covnbhd, \ \model(\covvaldum) = \sum_{j\in\featuresub}\modeldum\altind\feat{j}(\covvaldum\feat{j})\Big\}.
\]
For any \completename{} and \additivename{} $\expdef$ and \oraclehyptestname{} $\oraclehyptest$, 
\*[
    \specsym_{\expdef,\covdist,\covval}(\oraclehyptest) \leq 1 -
    \senssym_{\expdef,\covdist,\covval}(\oraclehyptest).
\]
\end{theorem}

Before we prove this, we need the following technical result, which is the multivariate analogue of \cref{fact:interior-freedom}.
\begin{theorem}\label{fact:interior-freedom-multivariate}
Fix any $\covval\in\covspace$, $\featuresub\subseteq[\covdim]$, 
$\inscale \in\PosReals^{\abs{\featuresub}}$,
$\covdist \in \probspace(\covspace)$, 
and
$(\modeldum\feat{j}:[\covval\feat{j}-\inscale\feat{j}, \covval\feat{j}+\inscale\feat{j}]\to\dataspace)_{j\in\featuresub}$.
Suppose that \cref{assn:inscale-covdist-multivariate} is satisfied and \cref{assn:piecewise-multivariate} is satisfied with $m=2^{\abs{\featuresub}}$.
For every $\expval \in \Reals^{\abs{\featuresub}\times\datadim}$, there exists 
$\model\in\modelspace$ such that for every \completename{} and \additivename{} \fielddefname{}, $\expdef(\model,\covdist,\covval)\feat{\featuresub} = \expval$, and if $\abs[0]{\covval\feat{j}-\covvaldum\feat{j}} \leq \inscale\feat{j}$ for all $j\in\featuresub$ then $\model(\covvaldum) = \sum_{j\in\featuresub} \modeldum\feat{j}(\covvaldum\feat{j})$.
\end{theorem}

We are now able to prove our main theorem.

\begin{proof}[Proof of \cref{fact:oracle-hypothesis-multivariate}]
For $j\not\in \featuresub$, let $\modeldum\nullind\feat{j} = \modeldum\altind\feat{j} \equiv 0$.
Fix $\expval = \mathbf{0}$.
Let $\model\nullind$ be the \modelname{} guaranteed to exist from \cref{fact:interior-freedom-multivariate} for $(\modeldum\nullind\feat{j})_{j\in[\covdim]}$, and similarly define $\model\altind$. Since $\model\nullind,\model\altind\in\modelspace$ by \cref{assn:piecewise-multivariate},
\*[
    \specsym_{\expdef,\covdist,\covval}(\oraclehyptest) 
    &= \inf_{\model\in\modelspace\nullind} [1-\oraclehyptest(\expdef(\model,\covdist,\covval))] \\
    &\leq 1-\oraclehyptest(\expdef(\model\nullind,\covdist,\covval)) \\
    &= 1-\oraclehyptest(\mathbf{0}) \\
    &= 1-\oraclehyptest(\expdef(\model\altind,\covdist,\covval)) \\
    &\leq 1-\inf_{\model\in\modelspace\altind} \oraclehyptest(\expdef(\model,\covdist,\covval)) \\
    &= 1 - \senssym_{\expdef,\covdist,\covval}(\oraclehyptest).
\]
\end{proof}

\subsection{Proof of \cref{fact:interior-freedom-multivariate}}

The proof follows by explicitly constructing a piecewise \modelname{} $\model$ that satisfies the desired properties and is linear outside of the neighbourhood around $\covval$.
In particular, for each $j\in\featuresub$ and $k\in[\datadim]$, we construct $\modelleft\feat{jk}$ and $\modelright\feat{jk}$, define
\*[
    \model\feat{jk}(\covvaldum\feat{j})
    =
    \begin{cases}
    \modelleft\feat{jk}(\covvaldum\feat{j})
    &\covvaldum\feat{j} \in (\covleft\feat{j}, \covval\feat{j}-\inscale\feat{j}) \\
    \modeldum\feat{j}(\covvaldum\feat{j})\feat{k}
    &\covvaldum\feat{j} \in [\covval\feat{j}-\inscale\feat{j}, \covval\feat{j}+\inscale\feat{j}] \\
    \modelright\feat{jk}(\covvaldum\feat{j})
    &\covvaldum\feat{j} \in (\covval\feat{j}+\inscale\feat{j},\covright\feat{j}) \\
    0
    & \text{otherwise},
    \end{cases}
\]
and finally define
\*[
    \model(\covvaldum)\feat{k}
    = \sum_{j\in\featuresub} \model\feat{jk}(\covvaldum\feat{j}).
\]
By definition, $\model(\covvaldum) = \sum_{j\in\featuresub} \modeldum\feat{j}(\covvaldum\feat{j})$ if $\abs[0]{\covvaldum\feat{j}-\covval\feat{j}}\leq\inscale\feat{j}$ for all $j\in\featuresub$.
Further, since $\expdef$ is \additivename{}, 
\*[
    \expdef(\model,\covval,\covdist)\feat{jk}
    = \expdef(\model\feat{jk}, \covval\feat{j}, \covdist\feat{j}).
\]
For some $\linparamleft\feat{jk},\linparamright\feat{jk}\in\Reals$ to be chosen in the proof, set
\*[
    \modelleft\feat{jk}(\covvaldum\feat{j})
    = 
    \linparamleft\feat{jk} \cdot (\covvaldum\feat{j}-\covval\feat{j}+\inscale\feat{j}) + \modeldum\feat{j}(\covval\feat{j}-\inscale\feat{j})\feat{k}
\]
and
\*[
    \modelright\feat{jk}(\covvaldum\feat{j})
    = 
    \linparamright\feat{jk} \cdot (\covvaldum\feat{j}-\covval\feat{j}-\inscale\feat{j}) + \modeldum\feat{j}(\covval\feat{j}+\inscale\feat{j})\feat{k}.
\]
These functions are piecewise linear on $(\covleft\feat{j}, \covval\feat{j}-\inscale\feat{j})$ and $(\covval\feat{j}+\inscale\feat{j},\covright\feat{j})$ for each $j\in\featuresub$ respectively, and hence $\model$ is piecewise linear on the product of these intervals (which form convex polytopes). Thus, by \cref{assn:piecewise}, $\model\in\modelspace$.
It remains to construct $\linparamleft\feat{jk}$ and $\linparamright\feat{jk}\in\Reals$ so that $\expdef(\model\feat{jk}, \covval\feat{j}, \covdist\feat{j}) = \expval\feat{jk}$ for each $j\in\featuresub$ and $k\in[\datadim]$.

Fix $j\in[\covdim]$ and $k\in[\datadim]$.
Recall that $\mu_j$ is a probability measure, and hence maps intervals of the form $(a,b)$ to a number in $[0,1]$.

For any $\linparamleft\feat{jk},\linparamright\feat{jk}\in\Reals$ that satisfy
\*[
    &\hspace{-1em}\linparamleft\feat{jk} \Bigg[ \EE_{\covobs\feat{j}\sim\covdist\feat{j}}\Big[\covobs\feat{j} \cdot\ind{\covobs\feat{j} \in (\covleft\feat{j}, \covval\feat{j}-\inscale\feat{j})}\Big] - (\covval\feat{j}-\inscale\feat{j}) \cdot \covdist\feat{j}\Big((\covleft\feat{j}, \covval\feat{j}-\inscale\feat{j})\Big) \Bigg] \\
    &\quad+ \linparamright\feat{jk} \Bigg[ \EE_{\covobs\feat{j}\sim\covdist\feat{j}}\Big[\covobs\feat{j} \cdot\ind{\covobs\feat{j} \in (\covval\feat{j}+\inscale\feat{j},\covright\feat{j})}\Big] - (\covval\feat{j}+\inscale\feat{j}) \cdot \covdist\feat{j}\Big((\covval\feat{j}+\inscale\feat{j},\covright\feat{j})\Big) \Bigg] \\
    &= -\modeldum\feat{j}(\covval\feat{j}-\inscale\feat{j})\feat{k} \cdot \covdist\feat{j}\Big((\covleft\feat{j}, \covval\feat{j}-\inscale\feat{j})\Big)
    -\modeldum\feat{j}(\covval\feat{j}+\inscale\feat{j})\feat{k} \cdot \covdist\feat{j}\Big((\covval\feat{j}+\inscale\feat{j},\covright\feat{j})\Big) \\
    &\quad- \EE_{\covobs\feat{j}\sim\covdist\feat{j}}\Big[\modeldum\feat{j}(\covobs\feat{j})\feat{k} \cdot\ind{\covobs\feat{j} \in[\covval\feat{j}-\inscale\feat{j},\covval\feat{j}+\inscale\feat{j}]}\Big]
    + \modeldum\feat{j}(\covval\feat{j})\feat{k} - \expval\feat{jk},
\]
since $\expdef$ is \completename{},
\*[
    \expdef(\model\feat{jk}, \covval\feat{j}, \covdist\feat{j})
    = \modeldum\feat{j}(\covval\feat{j})\feat{k} - \EE_{\covobs\feat{j}\sim\covdist\feat{j}}[\model\feat{jk}(\covobs\feat{j})]
    = \expval\feat{jk}.
\]

For simplicity, we set either $\linparamleft\feat{jk}$ or $\linparamright\feat{jk}$ to zero (at worst, this inflates the Lipschitz parameter by a factor of 2). 
By assumption (i),
\*[
    \max\Big\{\covdist\feat{j}\Big((\covleft\feat{j}, \covval\feat{j}-\inscale\feat{j})\Big), \covdist\feat{j}\Big((\covval\feat{j}+\inscale\feat{j},\covright\feat{j})\Big) \Big\} > 0,
\]
so one of 
\*[
    \EE_{\covobs\feat{j}\sim\covdist\feat{j}}\Big[\covobs\feat{j} \cdot\ind{\covobs\feat{j} \in (\covleft\feat{j}, \covval\feat{j}-\inscale\feat{j})}\Big] - (\covval\feat{j}-\inscale\feat{j}) \cdot \covdist\feat{j}\Big((\covleft\feat{j}, \covval\feat{j}-\inscale\feat{j})\Big)
\]
and
\*[
    \EE_{\covobs\feat{j}\sim\covdist\feat{j}}\Big[\covobs\feat{j} \cdot\ind{\covobs\feat{j} \in (\covval\feat{j}+\inscale\feat{j},\covright\feat{j})}\Big] - (\covval\feat{j}+\inscale\feat{j}) \cdot \covdist\feat{j}\Big((\covval\feat{j}+\inscale\feat{j},\covright\feat{j})\Big)
\]
are non-zero. 
If 
\*[
    &\hspace{-1em}\abs{\EE_{\covobs\feat{j}\sim\covdist\feat{j}}\Big[\covobs\feat{j} \cdot\ind{\covobs\feat{j} \in (\covleft\feat{j}, \covval\feat{j}-\inscale\feat{j})}\Big] - (\covval\feat{j}-\inscale\feat{j}) \cdot \covdist\feat{j}\Big((\covleft\feat{j}, \covval\feat{j}-\inscale\feat{j})\Big)} \\
    &>
    \abs{\EE_{\covobs\feat{j}\sim\covdist\feat{j}}\Big[\covobs\feat{j} \cdot\ind{\covobs\feat{j} \in (\covval\feat{j}+\inscale\feat{j},\covright\feat{j})}\Big] - (\covval\feat{j}+\inscale\feat{j}) \cdot \covdist\feat{j}\Big((\covval\feat{j}+\inscale\feat{j},\covright\feat{j})\Big)},
\]
set $\linparamright\feat{jk}=0$, and otherwise set $\linparamleft\feat{jk}=0$.
\manualendproof

\section{Proofs for \downtaskNAMEs{}}\label{sec:application-proofs}

\subsection{Proof of \cref{fact:local-stability-grad}}
Since $\grad \model(\covval)$ exists,
\*[
    \grad_j \model(\covval)
    = \lim_{\alpha \to 0} \frac{\model(\covval+\alpha e_j) - \model(\covval)}{\alpha}.
\]
Thus, for any $\outscale'>0$ there exists $\inscale>0$ such that for all $\alpha \leq \inscale$,
\*[
    \frac{\model(\covval+\alpha e_j) - \model(\covval)}{\alpha} 
    &\in
    \Big[
    \alpha \grad_j \model(\covval) - \alpha \outscale/4,
    \alpha \grad_j \model(\covval) + \alpha \outscale/4
    \Big] \\
    &\subseteq
    \Big[
    \alpha \grad_j \model(\covval) - \inscale \outscale/4,
    \alpha \grad_j \model(\covval) + \inscale \outscale/4
    \Big] \\
    &:= S(\alpha).
\]
Let $\oraclehyptest = 1-\ind{\sup_{\alpha\leq\inscale} \sup_{b \in S(\alpha)} \abs{b} \leq \inscale \outscale}$. When $\model \in \modelspace\altind$, then there exists $\alpha \leq \inscale$ such that for some $b \in S(\alpha)$, $\abs{b} = \abs{\model(\covval+\alpha e_j) - \model(\covval)} > \inscale\outscale$, and thus $\oraclehyptest = 1$.
Similarly, when $\model\in\modelspace\nullind$, then for every $\alpha \leq \inscale$ and every $b \in S(\alpha)$,
\*[
    \abs{b}
    \leq
    \abs{\model(\covval+\alpha e_j) - \model(\covval)} + \inscale\outscale/2
    \leq \inscale\outscale.
\]
This implies that $\oraclehyptest = 0$.

\manualendproof

\subsection{Proof of \cref{fact:local-instability}}

Define $\modeldum\nullind \equiv 0$ and $\modeldum\altind(\covvaldum\feat{j}) = \covvaldum\feat{j}-\covval\feat{j}$,
and let $\modelspacedum\nullind$ and $\modelspacedum\altind$ denote $\modelspace\nullind$ and $\modelspace\altind$ from \cref{fact:oracle-hypothesis} for the $\inscale$ specified in the statement of \cref{fact:local-instability}.
Clearly, $\modelspacedum\nullind \subseteq \modelspace\nullind$ and $\modelspacedum\altind\subseteq\modelspace\altind$ as respectively defined in \cref{fact:local-instability}.
For $\ell\neq j$, let $\modeldum\nullind\feat{\ell} = \modeldum\altind\feat{\ell} \equiv 0$.
By the assumption on $\modelspace$ in \cref{fact:local-instability}, \cref{assn:piecewise} is satisfied.
Thus, by \cref{fact:oracle-hypothesis},
\*[
    \specsym_{\expdef,\covdist,\covval}(\oraclehyptest; \modelspace\nullind) 
    &= \inf_{\model\in\modelspace\nullind} [1-\oraclehyptest(\expdef(\model,\covdist,\covval))] \\
    &\leq \inf_{\model\in\modelspacedum\nullind} [1-\oraclehyptest(\expdef(\model,\covdist,\covval))] \\
    &= \specsym_{\expdef,\covdist,\covval}(\oraclehyptest; \modelspacedum\nullind) \\
    &\leq 1 - \senssym_{\expdef,\covdist,\covval}(\oraclehyptest;\modelspacedum\altind) \\
    &= 1-\inf_{\model\in\modelspacedum\altind} \oraclehyptest(\expdef(\model,\covdist,\covval)) \\
    &\leq 1-\inf_{\model\in\modelspace\altind} \oraclehyptest(\expdef(\model,\covdist,\covval)) \\
    &= 1 - \senssym_{\expdef,\covdist,\covval}(\oraclehyptest;\modelspace\altind).
\]

\manualendproof

\subsection{Proof of \cref{fact:random-quadratic}}
By \completenames{} and \additivenames{}, taking $\covval=1$ gives $\expdef(\model,\covdist,\covval) = a (\covval^n - \EE_{\covdist} \covobs^n) - (\covval-\EE_{\covdist} \covobs) = a(1-\EE_{\covdist} \covobs^n) - 1$ while $\model'(\covval) = n a \covval^{n-1} - 1$. Thus, $\expdef(\model,\covdist,\covval) > 0$ if and only if $a > 1/(1-\EE_{\covdist} \covobs^n)$ while $\model'(\covval) > 0$ if and only if $a>1/n$. This implies that
\*[
    \PP_{\model\sim\pi}\Big[\sgn(\expdef(\model,\covdist,\covval)) \neq \sgn(\model'(\covval)) \Big]
    &= \PP_{a \sim \normaldist(0,1)}[a \in (1/n,1/(1-\EE_{\covdist} \covobs^n))]\\
    &\geq \PP_{a \sim \normaldist(0,1)}[a \in (1/2,2)]\\
    &\approx 0.2858,
\]
where the approximation is of the Gaussian CDF.
The upper bound is because $\PP_{a \sim \normaldist(0,1)}[a > 1/n] \leq 0.5$.
\manualendproof

\subsection{Proofs for Recourse and Spurious Features}

We first state the multivariate versions of these definitions for the most generality.

\begin{definition}[Recourse]\label{defn:recourse-multivariate}
Fix any $\covval\in\covspace$, $\featuresub\subseteq[\covdim]$, $k\in[\datadim]$,
$\inscale \in\PosReals^{\abs{\featuresub}}$, and
$\covdistdum \in \probspace(\covspace)$.
For any $\randsign\in\{\pm1\}^{\abs{\featuresub}}$,
let $\recoursetask\feat{\featuresub k}(\covval,\covdistdum,\inscale,\randsign)$ be defined by
\*[
    \modelspace\nullind
    &= \Big\{
    \model\in\modelspace\setdelim \EE_{\covobs\sim\covdistdum}\Big[\model(\covval\feat{[\covdim]\setminus\featuresub}, \covobs\feat{\featuresub})\feat{k} \condsym \covobs\feat{\featuresub} \in\prod_{j\in\featuresub} [\covval\feat{j}, \covval\feat{j}+\randsign\feat{j}\inscale\feat{j}]\Big] \\
    &\qquad\qquad\qquad\qquad> \EE_{\covobs\sim\covdistdum}\Big[\model(\covval\feat{[\covdim]\setminus\featuresub}, \covobs\feat{\featuresub})\feat{k} \condsym \covobs\feat{\featuresub} \in\prod_{j\in\featuresub} [\covval\feat{j}, \covval\feat{j}-\randsign\feat{j}\inscale\feat{j}]\Big]\Big\} \\
    \modelspace\altind
    &= \modelspace \setminus \modelspace\nullind.
\]
\end{definition}

\begin{definition}[Spurious \covNames{}]\label{defn:spurious-multivariate}
Fix any $\covval\in\covspace$, $\featuresub\subseteq[\covdim]$, $k\in[\datadim]$,
$\inscale \in\PosReals^{\abs{\featuresub}}$, and $\outscale>0$.
For any $\randsign\in\{\pm1\}^{\abs{\featuresub}}$,
let $\spurioustask\feat{\featuresub k}(\covval,\inscale,\randsign,\outscale)$ be defined by
\*[
    \modelspace\nullind &= \Big\{\model\in\modelspace \setdelim \sup_{\covvaldum\in\prod_{j\in\featuresub}(\covval\feat{j},\covval\feat{j}+\randsign\feat{j}\inscale\feat{j}] \times \covval\feat{[\covdim]\setminus\featuresub}} \model(\covvaldum)\feat{k} \leq 0\Big\} \\
    \modelspace\altind &= \Big\{\model\in\modelspace \setdelim 
    \sup_{\covvaldum\in\prod_{j\in\featuresub}(\covval\feat{j},\covval\feat{j}+\randsign\feat{j}\inscale\feat{j}] \times \covval\feat{[\covdim]\setminus\featuresub}} \model(\covvaldum)\feat{k} \geq \outscale \Big\}.
\]
\end{definition}

\begin{corollary}[Generalization of \cref{fact:oracle-hypothesis} for Recourse]\label{fact:oracle-hypothesis-recourse-multivariate}
Fix any $\covval\in\covspace$, $\featuresub\subseteq[\covdim]$,
$\inscale \in\PosReals^{\abs{\featuresub}}$,
and $\covdist\in\probspace(\covspace)$ such that \cref{assn:inscale-covdist-multivariate} is satisfied for each $j\in\featuresub$.
Fix
$k\in[\datadim]$,
$\covdistdum \in \probspace(\covspace)$,
and $\randsign\in\{\pm1\}^{\abs{\featuresub}}$,
and
let $\modelspace\nullind$ and $\modelspace\altind$ be defined by $\recoursetask\feat{\featuresub k}(\covval,\covdistdum,\inscale,\randsign)$.
Suppose that there exists $\model\nullind\in\modelspace\nullind$ and $\model\altind\in\modelspace\altind$ that each are of the form $\sum_{j\in\featuresub} \modeldum\feat{j}$ and each satisfy \cref{assn:piecewise-multivariate}.
Then
for any \completename{} and \additivename{} \fielddefname{} $\expdef$ and \oraclehyptestname{} $\oraclehyptest$,
\*[
    \specsym_{\expdef,\covdist,\covval}(\oraclehyptest) \leq 1 -
    \senssym_{\expdef,\covdist,\covval}(\oraclehyptest).
\]
\end{corollary}
\begin{proof}[Proof of \cref{fact:oracle-hypothesis-recourse-multivariate}]
This proof will benefit from slightly more detailed notation for $\specsym$ and $\senssym$. Specifically, we denote the explicit dependence on $\modelspace\nullind$ and $\modelspace\altind$. Let $\modelspace\leftind$ and $\modelspace\rightind$ denote $\modelspace\nullind$ and $\modelspace\altind$ respectively from \cref{defn:recourse-multivariate}. Pick the $\modeldum\nullind:\Reals\to\Reals$ such that
\*[
    &\hspace{-1em}\EE_{\covobs\sim\covdistdum}\Big[\modeldum\nullind(\covval\feat{[\covdim]\setminus\featuresub}, \covobs\feat{\featuresub})\feat{k} \condsym \covobs\feat{\featuresub} \in\prod_{j\in\featuresub} [\covval\feat{j}, \covval\feat{j}+\randsign\feat{j}\inscale\feat{j}]\Big] \\
    &\qquad\qquad\qquad\qquad> \EE_{\covobs\sim\covdistdum}\Big[\modeldum\nullind(\covval\feat{[\covdim]\setminus\featuresub}, \covobs\feat{\featuresub})\feat{k} \condsym \covobs\feat{\featuresub} \in\prod_{j\in\featuresub} [\covval\feat{j}, \covval\feat{j}-\randsign\feat{j}\inscale\feat{j}]\Big]
\]
as prescribed in the statement of \cref{fact:oracle-hypothesis-recourse-multivariate}.
Let $\modelspace\nullind$ be defined as in \cref{fact:oracle-hypothesis} for $\modeldum\nullind$. Clearly, $\modelspace\nullind\subseteq \modelspace\leftind$. Similarly, pick the $\modeldum\altind:\Reals\to\Reals$ such that
\*[
    &\hspace{-1em}\EE_{\covobs\sim\covdistdum}\Big[\modeldum\altind(\covval\feat{[\covdim]\setminus\featuresub}, \covobs\feat{\featuresub})\feat{k} \condsym \covobs\feat{\featuresub} \in\prod_{j\in\featuresub} [\covval\feat{j}, \covval\feat{j}+\randsign\feat{j}\inscale\feat{j}]\Big] \\
    &\qquad\qquad\qquad\qquad\leq \EE_{\covobs\sim\covdistdum}\Big[\modeldum\altind(\covval\feat{[\covdim]\setminus\featuresub}, \covobs\feat{\featuresub})\feat{k} \condsym \covobs\feat{\featuresub} \in\prod_{j\in\featuresub} [\covval\feat{j}, \covval\feat{j}-\randsign\feat{j}\inscale\feat{j}]\Big]
\]
as prescribed in the statement.
Define $\modelspace\altind$ accordingly, and observe that $\modelspace\altind\subseteq\modelspace\rightind$. 
By assumption, \cref{assn:piecewise-multivariate} is satisfied, and thus we can apply \cref{fact:oracle-hypothesis-multivariate} to obtain
\*[
    \specsym_{\expdef,\covdist,\covval}(\oraclehyptest; \modelspace\leftind, \modelspace\rightind) 
    &= \inf_{\model\in\modelspace\leftind} [1-\oraclehyptest(\expdef(\model,\covdist,\covval))] \\
    &\leq \inf_{\model\in\modelspace\nullind} [1-\oraclehyptest(\expdef(\model,\covdist,\covval))] \\
    &= \specsym_{\expdef,\covdist,\covval}(\oraclehyptest; \modelspace\nullind, \modelspace\altind) \\
    &\leq 1 - \senssym_{\expdef,\covdist,\covval}(\oraclehyptest; \modelspace\nullind, \modelspace\altind) \\
    &= 1-\inf_{\model\in\modelspace\altind} \oraclehyptest(\expdef(\model,\covdist,\covval)) \\
    &\leq 1-\inf_{\model\in\modelspace\rightind} \oraclehyptest(\expdef(\model,\covdist,\covval)) \\
    &= 1 - \senssym_{\expdef,\covdist,\covval}(\oraclehyptest;\modelspace\leftind, \modelspace\rightind).
\]
\end{proof}

\begin{corollary}[Generalization of \cref{fact:oracle-hypothesis} for Spurious \covNames{}]\label{fact:oracle-hypothesis-spurious-multivariate}
Fix any $\covval\in\covspace$, $\featuresub\subseteq[\covdim]$,
$\inscale \in\PosReals^{\abs{\featuresub}}$,
and $\covdist\in\probspace(\covspace)$ such that \cref{assn:inscale-covdist-multivariate} is satisfied for each $j\in\featuresub$.
Fix
$k\in[\datadim]$,
$\outscale>0$,
and $\randsign\in\{\pm1\}^{\abs{\featuresub}}$,
and
let $\modelspace\nullind$ and $\modelspace\altind$ be defined by $\spurioustask\feat{\featuresub k}(\covval,\inscale,\randsign,\outscale)$.
Suppose that there exists $\model\nullind\in\modelspace\nullind$ and $\model\altind\in\modelspace\altind$ that each are of the form $\sum_{j\in\featuresub} \modeldum\feat{j}$ and each satisfy \cref{assn:piecewise-multivariate}.
Then,
for any \completename{} and \additivename{} \fielddefname{} $\expdef$ and \oraclehyptestname{} $\oraclehyptest$, 
\*[
    \specsym_{\expdef,\covdist,\covval}(\oraclehyptest) \leq 1 -
    \senssym_{\expdef,\covdist,\covval}(\oraclehyptest).
\]
\end{corollary} 
\begin{proof}[Proof of \cref{fact:oracle-hypothesis-spurious-multivariate}]
The proof is identical to that of \cref{fact:oracle-hypothesis-recourse-multivariate}, just with $\modeldum\nullind$ and $\modeldum\altind$ defined appropriately to match the \behaviourname{} described by \cref{defn:spurious-multivariate} (e.g., constant 0 and constant $\outscale$ would work).
\end{proof}

\section{Proofs for Brute-Force \modelName{} Evaluations}\label{sec:sample-complexity-proofs}

\subsection{Notation for Sample Complexity}

We begin with additional notation to formalize brute-force \modelname{} evaluations.
A \emph{\queryalgoname{}} is any way to sequentially choose \obsnames{} to \queryname{} the \modelname{} at.
Let $\queryspace = \covspace\times\dataspace\times\Reals^{\covdim\times\datadim}$.
Formally, a \queryalgoname{} 
is any sequence of functions $\queryalgo = (\queryalgo\timeind{t})_{t\in\Nats}$ such that for each $t$, 
\*[
    \queryalgo\timeind{t}: \queryspace^{t-1} \to \probspace(\covspace).
\]
For any $\model\in\modelspace$, we write $\covobs\timeind{1:n} \sim \queryalgo\timeind{1:n}[\model]$ if 
\*[
    \covobs\timeind{1} &\sim \queryalgo\timeind{1} \quad \text{ and } \\
    \forall t>1 \qquad \covobs\timeind{t} \Bigcondsym \Big(\covobs\timeind{1:t-1}, \model(\covobs\timeind{1:t-1}),
    \grad\model(\covobs\timeind{1:t-1})\Big)
    &\sim \queryalgo\Big(\covobs\timeind{1:t-1}, \model(\covobs\timeind{1:t-1}),
    \grad\model(\covobs\timeind{1:t-1})\Big).
\]
We now introduce \emph{\samphyptestnaming{}}, where the \pracname{} makes a decision using only (random) \querynames{} of a \fielddefname{}. In practice, this is more general than \oraclehyptestnaming{}, since a \pracname{} can only access the output of a \fielddefname{} via \querynames{}.
Formally, a \samphyptestname{} is any function
\*[
    \samphyptest: \subspaceset{\queryspace} \to [0,1].
\]
This time, the output of $\samphyptest$ is \emph{the probability that the \pracname{} rejects the \nullname{}, conditional on the observed \querynames{}}. To study \specname{} and \sensname{} for the entire procedure of selecting \querynames{} and then performing a \samphyptestname{}, we introduce the $\{0,1\}$--valued random variable $\samphyprv{\queryalgo[\model]}{\samphyptest}{n}$, 
which has the distribution defined by
\*[
    \samphyprv{\queryalgo[\model]}{\samphyptest}{n}
    \Bigcondsym \covobs\timeind{1:n}
    &\sim \bernoullidist\Big(\samphyptest\Big(\covobs\timeind{1:n}, \model(\covobs\timeind{1:n}),
    \grad\model(\covobs\timeind{1:n})\Big)\Big) \\
    \covobs\timeind{1:n} &\sim \queryalgo\timeind{1:n}[\model].
\]
Finally,
for a fixed
$\queryalgo$, $\samphyptest$, and $\modelspace\nullind,\modelspace\altind \subseteq \modelspace$, 
\*[
    \sampspecsym\timeind{n}(\queryalgo, \samphyptest) &= 
    \inf_{\model\in\modelspace\nullind}
    \Big[1-\EE\Big[\samphyprv{\queryalgo[\model]}{\samphyptest}{n}\Big] \Big] \ \text{ and } \\
    \sampsenssym\timeind{n}(\queryalgo, \samphyptest) &=
    \inf_{\model\in\modelspace\altind}
    \EE\Big[\samphyprv{\queryalgo[\model]}{\samphyptest}{n}\Big].
\]
Once again, the \pracname{}'s goal is to maximize these simultaneously.

\subsection{Sample Complexity Theorems}

We define the \emph{Lipschitz constant} for any $\model:\covspace\to\dataspace$, $\covnbhd \subseteq \covspace$, and $k\in[\datadim]$ by
\*[
    \lipdef\lipind{\covnbhd}(\model)\feat{k}
    = \sup_{\covval,\covvaldum\in\covnbhd} \frac{\abs{\model(\covval)\feat{k} - \model(\covvaldum)\feat{k}}}{\norm{\covval-\covvaldum}_\infty}.
\]

First, we show that the task shown to be impossible for \oraclehyptestnames{} in \cref{fact:oracle-hypothesis-spurious-multivariate} can be solved with sufficiently many \querynames{}.
For any $\inscale>0$, let $\queryalgo\feat{\inscale}$ be such that 
\*[
    \queryalgo\timeind{t}\feat{\inscale} \equiv \uniformdist\Big((0,\inscale]^\covdim\Big).
\]
Further, for any $\covprob\in[0,1]$, define 
\*[
    \samphyptestoracle\feat{\covprob}(\covval\timeind{1:n}, \dataval\timeind{1:n})
    = \covprob + (1-\covprob) \cdot \ind{\exists t\in[n] \stT \dataval\timeind{t}>0}.
\]

Then, we have the following formal version of \cref{fact:sufficient-samples-informal}.
\begin{theorem}\label{fact:sufficient-samples}
Fix arbitrary $\inscale,\outscale>0$ and $\lipconst > 0$.
Suppose that for all $\model\in\modelspace$, $\lipdef\lipind{(0,\inscale]^\covdim}(\model) \leq \lipconst$,
and let
\*[
    \modelspace\nullind &= \{\model\in\modelspace \setdelim \sup_{\covval\in(0,\inscale]^\covdim} \model(\covval) \leq 0\} \\
    \modelspace\altind &= \Big\{\model\in\modelspace \setdelim 
    \sup_{\covval\in(0,\inscale]^\covdim} \model(\covval) \geq \eps \Big\}.
\]
Then, for any $\covprob\in[0,1]$, 
\*[
    \sampspecsym\timeind{n}(\queryalgo\feat{\inscale}, \samphyptestoracle\feat{\covprob})
    &= 1 - \covprob \ \text{ and } \\
    \sampsenssym\timeind{n}(\queryalgo\feat{\inscale},\samphyptestoracle\feat{\covprob})
    &= 1 - (1-\covprob)\Big(1-\Big(\frac{2\outscale}{\lipconst\inscale}\Big)^\covdim\Big)^n.
\]
\end{theorem}

In particular, this implies that
\*[
    \sampspecsym\timeind{n}(\queryalgo\feat{\inscale}, \samphyptestoracle\feat{\covprob})
    &= 1 - \sampsenssym\timeind{n}(\queryalgo\feat{\inscale}, \samphyptestoracle\feat{\covprob}) +  (1-\covprob) \cdot \Big(1 - \Big(1-\Big(\frac{2\outscale}{\lipconst\inscale}\Big)^\covdim\Big)^n \Big).
\]
The following result shows that this rate is nearly tight (as can be seen by a first-order Taylor expansion of $(1-x)^n$).
\begin{theorem}\label{fact:necessary-samples}
In the same setting as \cref{fact:sufficient-samples},
if \cref{assn:piecewise-multivariate} holds for linear $\modeldum$
then for any \bothagnosticname{} $\queryalgo$, $\samphyptest$, and $n\in\Nats$, 
\*[
    \sampspecsym\timeind{n}(\queryalgo, \samphyptest)
    \leq 1 - \sampsenssym\timeind{n}(\queryalgo, \samphyptest) 
    + n\cdot \Big(\Big\lfloor\frac{\lipconst\inscale}{2\outscale}\Big\rfloor\Big)^{-\covdim}.
\]
\end{theorem}

\subsection{Proof of \cref{fact:sufficient-samples}}

First, for any $\model\nullind\in\modelspace\nullind$ and $\covval\timeind{1:n}$, observe that
\*[
    \EE\Big[\samphyprv{\queryalgo\feat{\inscale}[\model\nullind]}{\samphyptestoracle\feat{\covprob}}{n} \condsym \covval\timeind{1:n} \Big]
    &= \samphyptestoracle\feat{\covprob}\Big(\covval\timeind{1:n}, \model\nullind(\covval\timeind{1:n}) \Big) 
    = \covprob + (1-\covprob) \cdot \ind{\exists t\in[n] \stT \model\nullind(\covval\timeind{t})>0} 
    = \covprob.
\]
That is,
\*[
    \sampspecsym\timeind{n}(\queryalgo\feat{\inscale}, \samphyptestoracle\feat{\covprob})
    = \inf_{\model\in\modelspace\nullind}
    \Big[1-\EE\Big[\samphyprv{\queryalgo\feat{\inscale}[\model]}{\samphyptestoracle\feat{\covprob}}{n}\Big] \Big]
    = 1-\covprob.
\]

Next, fix an arbitrary $\model\altind\in\modelspace\altind$. By definition, there exists $\covvalalt \in (0,\inscale]^\covdim$ such that $\model\altind(\covvalalt) \geq \outscale$. Since $\lipdef\lipind{(0,\inscale]^\covdim}(\model\altind) \leq \lipconst$, if $\norm{\covval-\covvalalt}_\infty < \outscale/\lipconst$ then $\model\altind(\covval) > 0$. Let $\cubealt = \{\covval\in(0,\inscale]^\covdim \setdelim \norm{\covval-\covvalalt}_\infty < \outscale/\lipconst\}$, 
\*[
    \sampcompevent\timeind{t} = \Big\{\covobs\timeind{t} \in \cubealt \Big\}, \quad\text{and}\quad
    \sampcompevent = \bigcup_{t\in[n]} \sampcompevent\timeind{t}.
\]
Further, let $\inoutratio = \lipconst \inscale / (2\outscale)$. 
By definition, since $\queryalgo\feat{\inscale}$ is independent of $\model\altind$,
\*[
    \PP\feat{\queryalgo\feat{\inscale}\timeind{1:n}}\Big[\comp{\sampcompevent} \Big]
    = (1-\inoutratio^{-\covdim})^n.
\]

Then,
\*[
    \sampsenssym\timeind{n}(\queryalgo\feat{\inscale}, \samphyptestoracle\feat{\covprob})
    &= \inf_{\model\in\modelspace\altind} \EE\Big[\samphyprv{\queryalgo\feat{\inscale}[\model]}{\samphyptestoracle\feat{\covprob}}{n}\Big] \\
    &= \inf_{\model\in\modelspace\altind}
    \Bigg(\EE\Big[\samphyprv{\queryalgo\feat{\inscale}[\model]}{\samphyptestoracle\feat{\covprob}}{n}\condsym \sampcompevent\Big]\PP\feat{\queryalgo\feat{\inscale}\timeind{1:n}}\Big[\sampcompevent\Big] + \EE\Big[\samphyprv{\queryalgo\feat{\inscale}[\model]}{\samphyptestoracle\feat{\covprob}}{n}\condsym \comp{\sampcompevent}\Big]\PP\feat{\queryalgo\feat{\inscale}\timeind{1:n}}\Big[\comp{\sampcompevent}\Big] \Bigg) \\
    &= \PP\feat{\queryalgo\feat{\inscale}\timeind{1:n}}\Big[\sampcompevent\Big] + \covprob \cdot \PP\feat{\queryalgo\feat{\inscale}\timeind{1:n}}\Big[\comp{\sampcompevent}\Big] \\
    &= 1 - (1-\covprob)(1-\inoutratio^{-\covdim})^n.
\]
\manualendproof

\subsection{Proof of \cref{fact:necessary-samples}}

Let $\model\nullind \equiv 0$ and $\PP\nullind$ denote $\PP_{\covobs\timeind{1:n}\sim\queryalgo\timeind{1:n}[\model\nullind]}$. 
Define $\lfloor \lipconst\inscale/(2\outscale) \rfloor = \inoutratio \in \Nats$.
Divide $(0,r \cdot (2\outscale/\lipconst)]^\covdim$ into $\inoutratio^\covdim$ cubes (each denoted by $\cube\feat{\ell}$ for $\ell\in[\inoutratio^\covdim]$) with sides that are left-open and right-closed of length $2\outscale/\lipconst$.
For each $\ell \in [\inoutratio^\covdim]$, let
\*[
    \sampcompevent\timeind{t}\feat{\ell} = \Big\{\covobs\timeind{t} \in \cube\feat{\ell} \Big\}
\]
and
\*[
    \sampcompevent\feat{\ell} = \bigcup_{t\in[n]} \sampcompevent\timeind{t}\feat{\ell}.
\]
By a union bound,
\*[
    \PP\nullind[\sampcompevent\feat{\ell}] \leq \sum_{t\in[n]} \PP\nullind[\sampcompevent\timeind{t}\feat{\ell}].
\]
Further, since $\cube\feat{\ell} \cap \cube\feat{\ell'} = \emptyset$ for all $\ell \neq \ell'$ and
\*[
    \bigcup_{\ell\in[\inoutratio^\covdim]} \cube\feat{\ell}
    \subseteq 
    (0,\inscale]^\covdim,
\]
\*[
    \sum_{\ell\in[\inoutratio^\covdim]} \sum_{t\in[n]} \PP\nullind[\sampcompevent\timeind{t}\feat{\ell}]
    \leq n.
\]
Thus, there must be some $\ell^*$ such that 
\*[
    \sum_{t\in[n]} \PP\nullind[\sampcompevent\timeind{t}\feat{\ell^*}]
    \leq \frac{n}{\inoutratio^\covdim}.
\]
Denote $\sampcompevent = \sampcompevent\feat{\ell^*}$ and $\sampcompevent\timeind{t} = \sampcompevent\timeind{t}\feat{\ell^*}$.

Define the $2L$-Lipschitz bump function taking maximal value $\eps$ by
\*[
    \model\altind(\covval)
    &=
    \begin{cases}
    0 &\covval \not\in\cube\feat{\ell^*} \\
    \text{square-based pyramid} &\covval\in\cube\feat{\ell^*}.
    \end{cases}
\]
Observe that $\model\altind\in\modelspace$ by \cref{assn:piecewise}.
Similarly to $\PP\nullind$, let $\PP\altind$ denote $\PP_{\covobs\timeind{1:n}\sim\queryalgo\timeind{1:n}[\model\altind]}$.
Since $\model\timeind{0} = \model\timeind{1}$ outside of $\cube\feat{\ell^*}$, it holds that
\*[
    \PP\nullind[\sampcompevent] = \PP\altind[\sampcompevent]
\]
and
\*[
    \samphyprv{\queryalgo[\model\nullind]}{\samphyptest}{n}
    \disteq
    \samphyprv{\queryalgo[\model\altind]}{\samphyptest}{n}
    \ \Bigcondsym \
    \comp{\sampcompevent}.
\]

Next,
\*[
    \EE\Big[\samphyprv{\queryalgo[\model\altind]}{\samphyptest}{n} \Big]
    &= 
    \EE\Big[\samphyprv{\queryalgo[\model\altind]}{\samphyptest}{n} \condsym \comp{\sampcompevent} \Big] \PP\altind[\comp{\sampcompevent}]
    +
    \EE\Big[\samphyprv{\queryalgo[\model\altind]}{\samphyptest}{n} \condsym \sampcompevent \Big] \PP\altind[\sampcompevent] \\
    &=
    \EE\Big[\samphyprv{\queryalgo[\model\nullind]}{\samphyptest}{n} \condsym \comp{\sampcompevent} \Big] \PP\nullind[\comp{\sampcompevent}]
    +
    \EE\Big[\samphyprv{\queryalgo[\model\altind]}{\samphyptest}{n} \condsym \sampcompevent \Big] \PP\nullind[\sampcompevent].
\]
Similarly,
\*[
    \EE\Big[\samphyprv{\queryalgo[\model\nullind]}{\samphyptest}{n} \Big]
    &=
    \EE\Big[\samphyprv{\queryalgo[\model\nullind]}{\samphyptest}{n} \condsym \comp{\sampcompevent} \Big] \PP\nullind[\comp{\sampcompevent}]
    +
    \EE\Big[\samphyprv{\queryalgo[\model\nullind]}{\samphyptest}{n} \condsym \sampcompevent \Big] \PP\nullind[\sampcompevent].
\]
Combining these, we obtain
\*[
    \sampspecsym\timeind{n}(\queryalgo, \samphyptest)
    &= \inf_{\model\in\modelspace\nullind}
    \Big[1-\EE\Big[\samphyprv{\queryalgo[\model]}{\samphyptest}{n}\Big] \Big] \\
    &\leq 1-\EE\Big[\samphyprv{\queryalgo[\model\nullind]}{\samphyptest}{n}\Big] \\
    &= 1 - \EE\Big[\samphyprv{\queryalgo[\model\altind]}{\samphyptest}{n}\Big] + \PP\nullind[\sampcompevent] \Big(\EE\Big[\samphyprv{\queryalgo[\model\altind]}{\samphyptest}{n}\condsym \sampcompevent\Big] - \EE\Big[\samphyprv{\queryalgo[\model\nullind]}{\samphyptest}{n} \condsym \sampcompevent\Big] \Big)\Big) \\
    &\leq 1 - \inf_{\model\in\modelspace\altind} \EE\Big[\samphyprv{\queryalgo[\model]}{\samphyptest}{n}\Big] + \PP\nullind[\sampcompevent] \\
    &= 1 - \sampsenssym\timeind{n}(\queryalgo, \samphyptest) + \frac{n}{\inoutratio^\covdim}.
\]
\manualendproof

\end{document}